\documentclass[letter]{article} % For LaTeX2e
\usepackage{nips15submit_e,times}
\usepackage{hyperref}
\usepackage{url}
\newcommand{\NipsArxiv}[2]{#1}
\renewcommand{\NipsArxiv}[2]{#2}
\pdfoutput=1
\usepackage{times}
\usepackage{graphicx} % more modern
\usepackage{subfigure} 
\usepackage{algorithm}
\usepackage{algorithmic}
\usepackage{hyperref}
\usepackage{mystyle}
\usepackage{notations}
\usepackage{natbib}
\renewcommand\cite{\citep}
\usepackage[labelfont=bf, width=0.95\textwidth]{caption}

\graphicspath{{figures/}}
\title{Biologically Inspired Dynamic Textures\\for Probing Motion Perception}
\newcommand{\Acknowledgments}{
We thank Guillaume Masson for useful discussions during the development of the experiments. We also thank Manon Bouy\'e and \'Elise Amfreville for proofreading.
LUP was supported by EC FP7-269921, ``BrainScaleS''.
The work of JV and GP was supported by the European Research Council (ERC project SIGMA-Vision). 
AIM and LUP were supported by SPEED ANR-13-SHS2-0006.
}
\author{
Jonathan Vacher\\
CNRS UNIC and Ceremade\\
Univ. Paris-Dauphine\\
75775 Paris Cedex 16, FRANCE \\
\texttt{vacher@ceremade.dauphine.fr} \\
\And
Andrew Isaac Meso \\
Institut de Neurosciences de la Timone \\
UMR 7289 CNRS/Aix-Marseille Universit\'e \\
13385 Marseille Cedex 05, FRANCE \\
\texttt{andrew.meso@univ-amu.fr} \\
\And
Laurent Perrinet \\
Institut de Neurosciences de la Timone\\
UMR 7289 CNRS/Aix-Marseille Universit\'e \\
13385 Marseille Cedex 05, FRANCE \\
\texttt{laurent.perrinet@univ-amu.fr} \\
\And
Gabriel Peyr\'e\\
CNRS and Ceremade\\
Univ. Paris-Dauphine\\
75775 Paris Cedex 16, FRANCE \\
\texttt{peyre@ceremade.dauphine.fr} 
}
\nipsfinalcopy % Uncomment for camera-ready version
\begin{document}
\maketitle
%
% !TEX root = ../MotionClouds-NIPS.tex

\begin{abstract}%
% general statement
Perception is often described as a predictive process based on an optimal inference with respect to a generative model. We study here the principled construction of a generative model specifically crafted to probe motion perception. 
% present method and aim
In that context, we first provide an axiomatic, biologically-driven derivation of the model. This model synthesizes random dynamic textures which are defined by stationary Gaussian distributions obtained by the random aggregation of warped patterns. 
% contributions
%\done{Laurent: shifted sPDE part to SM}% Moreover
%\done{Gab: ``inverting'' is a bit weird, I think we should mention the sPDE, adding ``(detailed in supplementary material)'' - Laurent: I tried to fix that.} 
Importantly, we show that this model can equivalently be described as a stochastic partial differential equation.
Using this characterization of motion in images, it allows us to recast motion-energy models into a principled Bayesian inference framework. 
% results
Finally, we apply these textures in order to psychophysically probe speed perception in humans. In this framework, while the likelihood is derived from the generative model, the prior is estimated from the observed results and accounts for the perceptual bias in a principled fashion.
% show the usefulness of this framework
\end{abstract}%

% !TEX root = ../MotionClouds-NIPS.tex

\section{Motivation}%Introduction: From the motion-energy model to Motion Clouds}

A normative explanation for the function of perception is to infer relevant hidden parameters from the sensory input with respect to a generative model~\citep{Gregory80}. Equipped with some prior knowledge about this representation, this corresponds to the \emph{Bayesian brain} hypothesis, as has been perfectly illustrated by the particular case of motion perception~\citep{Weiss02}. However, the Gaussian hypothesis related to the parameterization of knowledge in these models ---for instance in the formalization of the prior and of the likelihood functions--- does not always fit with psychophysical results~\citep{Wei12}. As such, a major challenge is to refine the definition of generative models so that they conform to the widest variety of results.  

%%%%
%\subsection{Previous Works}
From this observation, the estimation problem inherent to perception is linked to the definition of an adequate generative model. In particular, the simplest generative model to describe visual motion is the luminance conservation equation. It states that luminance  $I(x,t)$ for $(x,t) \in \RR^2\times \RR$ is approximately conserved along trajectories defined as integral lines of a vector field $v(x,t) \in \RR^2\times \RR$. The corresponding generative model defines random fields as solutions to the stochastic partial differential equation (sPDE), 
\eql{\label{eq-luminance}
	\dotp{v}{\nabla I} + \pd{I}{t}  = W,
} 
where $\dotp{\cdot}{\cdot}$ denotes the Euclidean scalar product in $\RR^2$, $\nabla I$ is the spatial gradient of $I$. To match the statistics of natural scenes or some category of textures, the driving term $W$ is usually defined as a colored noise corresponding to some average spatio-temporal coupling, and is parameterized by a covariance matrix $\Sigma$, while the field is usually a constant vector $v(x, t)=v_0$ accounting for a full-field translation with constant speed. 

Ultimately, the application of this generative model is essential for probing the visual system, for instance to understand how observers might detect motion in a scene.  %
Indeed, as shown by~\citep{Nestares00,Weiss02}, the negative log-likelihood corresponding to the luminance conservation model~\eqref{eq-luminance} and determined by a hypothesized speed $v_0$ is proportional to the value of the motion-energy model~\citep{Adelson85} 
%\eq{ % \label{energy}
	 $\norm{ \dotp{v_0}{\nabla (K \star I)} + \pd{(K \star I)}{t}  }^2$, 
%} 
where $K$ is the whitening filter corresponding to the inverse of $\Si$, and $\star$ is the convolution operator. Using some prior knowledge on the distribution of motions, for instance a preference for slow speeds, this indeed leads to a Bayesian formalization of this inference problem~\citep{Weiss01}. This has been successful in accounting for a large class of psychophysical observations~\citep{Weiss02}. As a consequence, such probabilistic frameworks allow one to connect different models from computer vision to neuroscience with a unified, principled approach.

%-------------------------------------------------------------%
%: fig: principle
\begin{figure}[b!]
\vspace{2mm}
\begin{center}
\includegraphics[width=.85\textwidth]{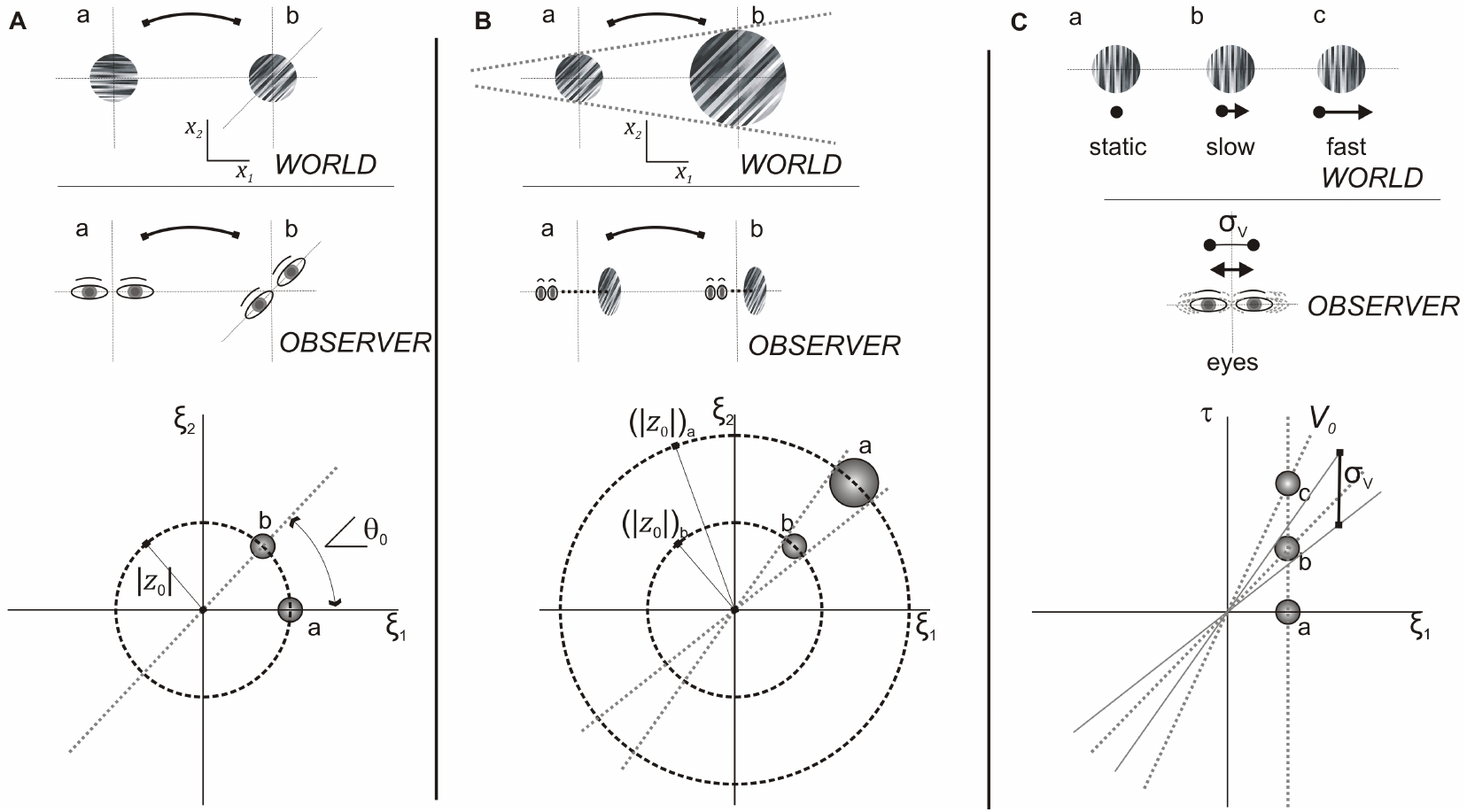}
\end{center}
\caption{\emph{Parameterization of the class of Motion Clouds stimuli.} The illustration relates the parametric changes in MC with real world (top row) and observer (second row) movements. \textbf{(A)} Orientation changes resulting in scene rotation are parameterized through $\th$ as shown in the bottom row where a horizontal $a$ and obliquely oriented $b$ MC are compared. \textbf{(B)} Zoom movements, either from scene looming or observer movements in depth, are characterised by scale changes reflected by a scale or frequency term $\z$ shown for a larger or closer object $b$ compared to more distant $a$. 
%\todo{Laurent to Andrew: Could $\sr$ be denoted by a vertical segment? this arc suggest that it is a rotation of the plane (to me...)}
\textbf{(C)} Translational movements in the scene characterised by $V$ using the same formulation for static (a) slow (b) and fast moving MC, with the variability in these speeds quantified by $\sr$. $(\xi$ and $\tau)$ in the third row are the spatial and temporal frequency scale parameters. The development of this formulation is detailed in the text. }
\label{fig:warps}
\end{figure}
%-------------------------------------------------------------%

However the model defined in~\eqref{eq-luminance} is obviously quite simplistic with respect to the complexity of natural scenes. It is therefore useful here to relate this problem to solutions proposed by texture synthesis methods in the computer vision community. Indeed, the literature on the subject of static textures synthesis is abundant (see~\citep{WLKT09} and the references therein for applications in computer graphics). %
Of particular interest for us is the work of~\citet{Galerne11}, which proposes a stationary Gaussian model restricted to static textures. % 
Realistic dynamic texture models are however less studied, and the most prominent method is the non-parametric Gaussian auto-regressive (AR) framework of~\citep{Doretto-Dyntex}, which has been refined in~\citep{2014-xia-siims}.%, and bears similarities with the discretized AR implementation of our parametric model (see Section~\ref{sec-numerics}). %

%%%%
\paragraph{Contributions.}

Here, we seek to engender a better understanding of motion perception by improving generative models for dynamic texture synthesis. From that perspective, we motivate the generation of optimal stimulation within a stationary Gaussian dynamic texture model. 
We base our model on a previously defined heuristic~\citep{Leon12,Simoncini12} coined ``Motion Clouds''. 
Our first contribution is an axiomatic derivation of this model, seen as a shot noise aggregation of dynamically warped ``textons''.
This formulation is important to provide a clear understanding of the effects of the model's parameters manipulated during psychophysical experiments. Within our generative model, they correspond to average translation speed and orientation of the ``textons'' and standard deviations of random fluctuations around this average.  
Our second contribution \NipsArxiv{(proved in the supplementary materials)}{} is to demonstrate an explicit equivalence between this model and a class of linear stochastic partial differential equations (sPDE). This shows that our model is a generalization of the well-known luminance conservation equation. 
This sPDE formulation has two chief advantages: it allows for a real-time synthesis using an AR recurrence and it allows one to recast the log-likelihood of the model as a generalization of the classical motion energy model, which in turn is crucial to allow for a Bayesian modeling of perceptual biases. 
Our last contribution is an illustrative application of this model to the psychophysical study of motion perception in humans. 
%\todo{Jonathan: there are no explicit model predictions making the above sentense misleading, remove or rephrase! \ldots}
This application shows how the model allows us to define a likelihood, which enables a simple fitting procedure to determine the prior driving the perceptual bias. 

%%%
\paragraph{Notations.}

In the following, we will denote $(x,t) \in \RR^2 \times \RR$ the space/time variable, and $(\xi,\tau) \in \RR^2 \times \RR$ the corresponding frequency variables. If $f(x,t)$ is a function defined on $\RR^3$, then $\hat f(\xi,\tau)$ denotes its Fourier transform. For $\xi \in \RR^2$, we denote $\xi = \norm{\xi}(\cos( \angle{\xi}),\sin(\angle{\xi})) \in \RR^2$ its polar coordinates. For a function $g$ in $\RR^2$, we denote $\bar g(x) = g(-x)$.  
In the following, we denote with a capital letter such as $A$ a random variable, a we denote $a$ a realization of $A$, we let $\dis_A(a)$ be the corresponding distribution of $A$. 

% !TEX root = ../MotionClouds-NIPS.tex

\section{Axiomatic Construction of a Dynamic Texture Stimulation Model}
\label{sec-axiomatic}
%  of the visual transformations that operated

%As described in the introduction, s % designing optimal motion-energy models
Solving a model-based estimation problem and finding optimal dynamic textures for stimulating an instance of such a model can be seen as equivalent mathematical problems. In the luminance conservation model~\eqref{eq-luminance}, the generative model is parameterized by a spatio-temporal coupling function, which is encoded in the covariance $\Sigma$ of the driving noise and the motion flow $\vz$.
This coupling (covariance) is essential as it quantifies the extent of the spatial integration area as well as the integration dynamics, an important issue in neuroscience when considering the implementation of integration mechanisms from the local to the global scale.
In particular, it is important to understand modular sensitivity in the various lower visual areas with different spatio-temporal selectivities such as Primary Visual Cortex (V1) or ascending the processing hierarchy, Middle Temple area (MT).
For instance, by varying the frequency bandwidth of such dynamic textures, distinct mechanisms for perception and action have been identified~\citep{Simoncini12}. 
However, such textures were based on a heuristic~\citep{Leon12},
and our goal here is to develop a principled, axiomatic definition.

%%%%%%%%%%%%%%%%%%%%%%%%%%%%%%%%%%%%%%%%%%%%%
%%%%%%%%%%%%%%%%%%%%%%%%%%%%%%%%%%%%%%%%%%%%%
\subsection{From Shot Noise to Motion Clouds}

We propose a mathematically-sound derivation of a general parametric model of dynamic textures. This model is defined by aggregation, through summation, of a basic spatial ``texton'' template $g(x)$. The summation reflects a transparency hypothesis, which has been adopted for instance in~\citep{Galerne11}. While one could argue that this hypothesis is overly simplistic and does not model occlusions or edges, it leads to a tractable framework of stationary Gaussian textures, which has proved useful to model static micro-textures~\citep{Galerne11} and dynamic natural phenomena~\citep{2014-xia-siims}. The simplicity of this framework allows for a fine tuning of frequency-based (Fourier) parameterization, which is desirable for the interpretation of psychophysical experiments.

We define a random field as 
\eql{\label{eq-deadleaves}
	I_\la(x,t) \eqdef  \frac{1}{\sqrt{\la}} \sum_{p \in \NN} g( \phi_{\Geom_p}( x-X_{p}  - V_p t ) )
}
where $\phi_\geom : \RR^2 \rightarrow \RR^2$ is a planar warping parameterized by a finite dimensional vector $\geom$. Intuitively, this model corresponds to a dense mixing of stereotyped, static textons as in~\citep{Galerne11}. The originality is two-fold. First, the components of this mixing are derived from the texton by visual transformations $\phi_{\Geom_p}$ which may correspond to arbitrary transformations such as zooms or rotations, illustrated in Figure 1. Second, we explicitly model the motion (position $X_p$ and speed $V_p$) of each individual texton. 
The parameters $(X_p,V_p,\Geom_p)_{p \in \NN}$ are independent random vectors. They account for the variability in the position of objects or observers and their speed, thus mimicking natural motions in an ambient scene. %
The set of translations $(X_p)_{p \in \NN}$ is a 2-D Poisson point process of intensity $\la>0$.
% , i.e., defining for any measurable $A$, $C(A) = \sharp\enscond{p}{X_p \in A}$, one has that $\EE(C(A))=\la |A|$ (where $|A|$ is the measure of $A$) and $C(A)$ is independent of $C(B)$ if $A \cap B = \emptyset$.  
The following section instantiates this idea and proposes canonical choices for these variabilities. The warping parameters $(\Geom_p)_p$ are distributed according to a distribution $\falpha$. The speed parameters $(V_p)_p$ are distributed according to a distribution $\fv$ on $\RR^2$.  
The following result shows that the model~\eqref{eq-deadleaves} converges to a stationary Gaussian field and gives the parameterization of the covariance.
Its proof follows from a specialization of~\citep[Theorem 3.1]{GalernePHD} to our setting. 

\begin{prop}\label{eq-conv-deadleaves}
	$I_\la$ is stationary with bounded second order moments. Its covariance is $\Si(x,t,x',t') = \ga(x-x',t-t')$ where $\ga$ satisfies 
	\eql{\label{eq-cov-proposition}
		\foralls (x,t) \in \RR^3, \quad
		\ga(x,t) = \int \int_{\RR^2} c_g(\phi_\geom(x-\speed t))  \fv(\speed) \falpha(\geom) \d \speed \d \geom
	} 
	where $c_g = g \star \bar g$ is the auto-correlation of $g$. 
	When $\la \rightarrow +\infty$, it converges (in the sense of finite dimensional distributions) toward a stationary Gaussian field $I$ of zero mean and covariance $\Si$. 
\end{prop}

%%%%%%%%%%%%%%%%%%%%%%%%%%%%%%%%%%%%%%%%%%%%%
%%%%%%%%%%%%%%%%%%%%%%%%%%%%%%%%%%%%%%%%%%%%%
\subsection{Definition of ``Motion Clouds''}

% \todo{Laurent: we may incorporate the fact that variability ads up for translations (extrinsic noise) and non-rotational head movements (intrinsic noise)} %

We detail this model here with warpings as rotations and scalings (see Figure~\ref{fig:warps}).  These account for the characteristic orientations and sizes (or spatial scales) in a scene with respect to the observer 
\eq{
	\foralls \geom = (\th,\z) \in [-\pi, \pi) \times \RR_+^{*}, \quad
	\phi_\geom(x) \eqdef \z R_{-\th}(x),  
}
where $R_\th$ is the planar rotation of angle $\th$. We now give some physical and biological motivation underlying our particular choice for the distributions of the parameters. We assume that the distributions $\fz$ and $\ftheta$ of spatial scales $\z$ and orientations $\th$, respectively (see Figure 1), are independent and have densities, thus considering 
\NipsArxiv{$ \foralls \geom = (\th,\z) \in [-\pi, \pi) \times \RR_+^{*}, \quad
	\falpha(\geom) = \fz(\z) \, \ftheta(\th). $}
{
\eq{
	\foralls \geom = (\th,\z) \in [-\pi, \pi) \times \RR_+^{*}, \quad
	\falpha(\geom) = \fz(\z) \, \ftheta(\th).
 }
}
%\todo{Laurent: shouldn't speed be simply be noted $v$ instead of $\nu$? Gab: I would prefer not because $v$ is used in the psychophysic part at horizontal (scalar) speed.}
The speed vector $\speed$ is assumed to be randomly fluctuating around a central speed $\vz$, so that
\eql{\label{eq-distr-V}
	\foralls \speed \in \RR^2, \quad \fv(\speed) = \fr(\norm{\speed - \vz}).
}
%
%
% \todo{Laurent : a dumb question : as when you go from cartesian to polar coordinates, the transformation include a radius term - why isn't it the case here?} \done{we just want $\int_{\geom}\d \falpha(\geom)=1$  }
%
%
% \todo{Laurent : say somewhere that bandwidth in speed orientation (angle) is redundant with that in freq orientation angle through the aperture problem - perhaps discussion?  especially in the limit of our  axiomatic  definition (alpha tends to 0) }
% \done{Jonathan I think it is possible to generalize with an arbitrary distribution $\fphi$ for speed angles we will see mathematically if it's really redundant}
%
% which has a characteristic scale $\sigma \in \RR$ over the spatial Fourier domain
%\eq{
%	\foralls \xi \in \RR^2, \quad
%	\hat g(\xi) = \frac{1}{\sqrt{\sigma}} \: \hat g_0\pa{ \frac{\xi - \xiz}{\sigma} }
%}
% where $g_0$ can be taken for instance as a Gaussian function $g_0(\xi) = e^{-\xi^2}$.
%
In order to obtain ``optimal'' responses to the stimulation (as advocated by~\citep{young01}), it makes sense to define the texton $g$ to be equal to an oriented Gabor acting as an atom, based on the structure of a standard receptive field of V1. Each would have a scale $\sigma$ and a central frequency $\xiz$. Since the orientation and scale of the texton is handled by the $(\th,\z)$ parameters, we can impose without loss of generality the normalization $\xiz = (1,0)$. 
In the special case where $\sigma \rightarrow 0$, $g$ is a grating of frequency $\xiz$, and the image $I$ is a dense mixture of drifting gratings, whose power-spectrum has a closed form expression detailed in Proposition~\ref{prop-mc-spectrum}. Its proof can be found \NipsArxiv{in the supplementary materials.}{in Section~\ref{sec-proof-prop-mc-spectrum}.} We call this Gaussian field a Motion Cloud (MC), and it is parameterized by the envelopes $(\fz,\ftheta,\fv)$ and has central frequency and speed $(\xiz,\vz)$. 
Note that it is possible to consider any arbitrary textons $g$, which would give rise to more complicated parameterizations for the power spectrum $\hat g$, but we decided here to stick to the simple case of gratings. 

% The following proposition shows that in the limit of small $\sigma$, the shot noise model, specialized to the setting where $\phi_\geom$ are rotations and zooms, defines a Gaussian random field with an explicitly parameterized power-spectrum. We call such a field a Motion Cloud (MC) parameterized by the envelopes $(\fz,\ftheta,\fv)$ and with a central frequency and speed $(\xiz,\vz)$. 

\begin{prop}\label{prop-mc-spectrum}
	When $g(x) = e^{\imath \dotp{x}{\xiz}}$, 
	the image $I$ defined in Proposition~\ref{eq-conv-deadleaves} is a stationary Gaussian field of covariance having the power-spectrum
	\eql{\label{eq-dfn-mc-spectrum}
		\foralls (\xi,\tau)  \in \RR^2 \times \RR, \:
		\hat \ga(\xi,\tau) = \frac{ \fz\pa{  \norm{\xi} } }{\norm{\xi}^2}
			\ftheta\pa{ \angle{\xi} }
%		    \Ll(\fr)\pa{ \frac{-\tau}{  \norm{\xi} } - \vzMod\cos(\vzAng- \angle{\xi} ) }	
			\Ll(\fr)
			\pa{ 
				-\frac{
					\tau + \dotp{\vz}{\xi}
				}{  
					\norm{\xi} 
				}
			}, 
	}
where the linear transform $\Ll$ is such that $\forall u \in \RR,  \Ll(f)(u) = \int_{-\pi}^{\pi} f( -u/ \cos(\phi) ) \d \phi$.  	
\end{prop}

\begin{rem}
Note that the envelope of $\hat\ga$ is shaped along a cone in the spatial and temporal domains. This is 
an important and novel contribution when compared to a Gaussian formulation like a classical Gabor. 
In particular, the bandwidth is then constant around the speed plane or the orientation line with respect to spatial frequency.  Basing the generation of the textures on all possible translations, rotations and zooms, we thus provide a principled approach to show that bandwidth should be proportional to spatial frequency to provide a better model of moving textures.
\end{rem}

%%%%%%%%%%%%%%%%%%%%%%%%%%%%%%%%%%%%%%%%%%%%%
%%%%%%%%%%%%%%%%%%%%%%%%%%%%%%%%%%%%%%%%%%%%%
\subsection{Biologically-inspired Parameter Distributions}

We now give meaningful specialization for the probability distributions $(\fz, \ftheta, \fr)$, which are inspired by some known scaling properties of the visual transformations relevant to dynamic scene perception.

First, small, centered, linear movements of the observer along the axis of view (orthogonal to the plane of the scene) generate centered planar zooms of the image. From  the linear modeling of the observer's displacement and the subsequent multiplicative nature of zoom, scaling should follow a Weber-Fechner law stating that subjective sensation when quantified is proportional to the logarithm of stimulus intensity. Thus, we choose the scaling $\z$ drawn from a log-normal distribution $\fz$, defined in~\eqref{eq:fz-fth}. The bandwidth $\sz$ quantifies the variance in the amplitude of zooms of individual textons relative to the set characteristic scale $\z_0$. 
Similarly, the texture is perturbed by variation in the global angle $\theta$ of the scene: for instance, the head of the observer may roll slightly around its normal position. The von-Mises distribution -- as a good approximation of the warped Gaussian distribution around the unit circle -- is an adapted choice for the distribution of $\theta$  with mean $\theta_0$ and bandwidth $\stheta$, see~\eqref{eq:fz-fth}. 
%
%
% \done{John: je propose d'enlever la phrase suivante (ou de la deplacer): Note that both axis (zoom and translations) correspond respectively to linear transformations in the log-polar representation that is often used for describing retinotopy. }
%
We may similarly consider that the position of the observer is variable in time. On first order, movements perpendicular to the axis of view dominate, generating random perturbations to the global translation $\vz$ of the image at speed $\speed-\vz \in \RR^2$. These perturbations are for instance described by a Gaussian random walk: take for instance tremors, which are constantly jittering, small ($\leq 1$ deg) movements of the eye. This justifies the choice of a radial distribution~\eqref{eq-distr-V} for $\fv$. 
This radial distribution $\fr$ is thus selected as a bell-shaped function of width $\sr$, and we choose here a Gaussian function for simplicity, see~\eqref{eq:fz-fth}.
Note that, as detailed in \NipsArxiv{the supplementary}{Section~\ref{sec-expression-fr}} a slightly different bell-function (with a more complicated expression) should be used to obtain an exact equivalence with the sPDE discretization mentioned in Section~\ref{sec-spde-model}.

The distributions of the parameters are thus chosen as
\eql{\label{eq:fz-fth}%\label{eq:speed_angle}
	\fz(\z) \propto \dfrac{\z_0}{\z} 
	e^{
		-\frac{\ln\left( \frac{\z}{\z_0} \right)^2}{2\ln\left(1 + \sz^2 \right)}
	}, 
%	\exp\pa{
%		-\dfrac{\ln\left( \frac{\z}{\z_0} \right)^2}{2\ln\left(1 + \sz^2 \right)}
%	}, 
	\quad
	\ftheta(\th) \propto e^{ \frac{\cos (2 ( \theta - \theta_0 ))}{4 \stheta^2} }
	\qandq
	\fr(r) \propto e^{-\frac{r^2}{ 2\sr^2 }}.
}
\begin{rem} Note that in practice we have parametrized $\fz$ by its mode $ m_{\Z} = \argmax_{\z} \fz(\z) $ and standard deviation $d_{\Z} = \sqrt{ \int \z^2 \fz(\z) \d \z }$, see \NipsArxiv{the supplementary material and \citep{Field87}}{Section~\ref{sec-param-fz} and \citep{Field87}}.
\label{rem:param} 
\end{rem}
%\todo{Jonathan: the expression $P_z(z)$ should be acurate and consistent with what we do in the code. Perhaps add in text clarification \ldots}
%------------------------ figure 2 --------------------------
% !TEX root = ../MotionClouds-NIPS.tex
\begin{figure}[h!]
\begin{center}
\begin{tikzpicture} 
\begin{scope}[scale=0.75]
 
  \begin{scope}[xshift=0cm]
 % projection on the speed plane
 %%%%%%%%%%%%%%%%%%%%%%%%%%%%%%%
 	\draw[thin,->,black!50] (-1.5,0) -- (2.,0);
    \draw[thin,->,black!50] (0,-1.75) -- (0,1.75);
    
\begin{scope}[cm={1,1,0,1,(0,0)}] % Sets the coordinate trafo matrix entries.
\shade[inner color=gray, outer color=white]   (-10:0.75cm) arc (-10:10:0.75cm)   --
  								 (10:1.25cm) arc (10:-10:1.25cm)  -- cycle;
\shade[inner color=gray, outer color=white]   (-10:-0.75cm) arc (-10:10:-0.75cm)   --
  								 (10:-1.25cm) arc (10:-10:-1.25cm)  -- cycle;
\draw[thin,<->,black!50] (-12:0.75cm) -- (-12:1.25cm);  %c'est obtenu par le shearing
\end{scope}

    	\draw[dashed] (1., 0) -- (1.,1.);	
  \node[anchor=north] at (1.,0.){$\z_0$};
  \node at (24:1.6){$\sigma_{\Z}$};

  \draw[thin,<->,black!50] (1.4, 1.6) -- (1.4, 1.1);
  \node at (32:2.5){$\sr$};

    	\draw[dashed] (-1.5,-1.5) -- (1.5,1.5);	

  % axis label
  \node[black!50] at (2.,-0.3){$\xi_1$};
  \node[black!50] at (-.25,1.75){$\tau$};
  
  \node at (-1.05,0.7){Slope: $\angle{\vz}$};

 \end{scope}

   \begin{scope}[xshift=-4.25cm]

  % projection on the spatial frequency plane
  %%%%%%%%%%%%%%%%%%%%%%%%%%%%%%%%%%%%%%%%%%%
        \draw[thin,->,black!50] (-1.75,0) -- (1.75,0);
        \draw[thin,->,black!50] (0,-1.75) -- (0,1.75);
        
        % remplacer par des gaussiennes ( pas de cone )
        \shade[inner color=red, outer color=white]
  ($(0,0) + (30:0.75cm)$) arc (30:60:0.75cm) 
  --
  ($(0,0) + (60:1.25cm)$) arc (60:30:1.25cm)
  -- cycle;
        \shade[inner color=red, outer color=white]
  ($(0,0) + (30:-0.75cm)$) arc (30:60:-0.75cm)
  --
  ($(0,0) + (60:-1.25cm)$) arc (60:30:-1.25cm)
  -- cycle;
  %%%%%%%%%%%%%%%%%%%%%%%%%%%%%%%%%%%%%%%%%%%%%%%%%%%%%
        \draw[thin,->,black!50]
  ($(0,0) + (0:0.5cm)$) arc (0:45:0.5cm);
  		\draw[dashed]
  ($(0,0) + (0:1cm)$) arc (0:360:1cm);
        \draw[thin,<->,black!50]
  ($(0,0) + (30:1.35cm)$) arc (30:60:1.35cm);
    	\draw[dashed] (-1.5,-1.5) -- (1.5,1.5);	
  		\draw[thin,->,black!50] (0,0) -- (120:1);
%  		\draw[thin,<->,black!50] (26:0.75) -- (26:1.25);
  		\draw[thin,<->,black!50] (22:0.75cm) -- (22:1.25cm);  %c'est obtenu par le shearing
%  		\draw[->,>=latex] (0.9,-1) to[bend left] (-0.7,-0.85); %TODO : c'ext quoi cette fleche?
  
  % axis label
  \node[black!50] at (-0.25,1.75){$\xi_2$};
  \node[black!50] at (1.75,-0.3){$\xi_1$};
  \node at (-22:0.4){$\th_0$};
  \node at (140:0.6){$\z_0$};
  \node at (35:1.7){$\sigma_{\Theta}$};
  \node at (18:1.6){$\sigma_{\Z}$};
%%   \node at (1.25,-1){$\frac{1}{||\text{f}||^{\geom}}$}; %%% REMOVED %%%
  
%%  \node at (3.8,0){$ h \circledast W =$}; %% REMOVED %%
 \end{scope}

  \draw[->] (2.3,0) -- (4.3,0);

  \node[anchor=north east]  at (3,-1.7){Two different projections of $\hat \ga$ in Fourier space};

    \begin{scope}[xshift=6cm]

  \node at (0,0.4){\includegraphics[scale=0.2]{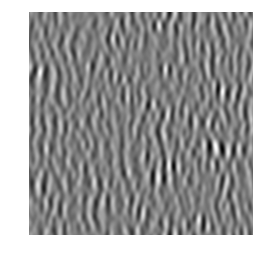}};
  \node at (.2,0.2){\includegraphics[scale=0.2]{mc.png}};
  \node at (.4,0){\includegraphics[scale=0.2]{mc.png}};
  \node at (.6,-0.2){\includegraphics[scale=0.2]{mc.png}};
  \draw[thin,->,black!50] (-1.4,-0.8) -- (-.6,-1.6);
  \node[black!50] at (-.4,-1.6){$t$};
 \end{scope}
  
    \begin{scope}[xshift=9.1cm]

  \node at (-.2,0.4){\includegraphics[scale=0.2]{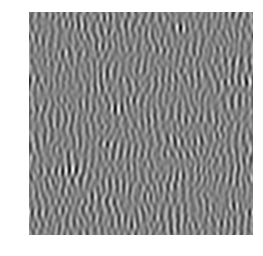}};
  \node at (0.,0.2){\includegraphics[scale=0.2]{mcmc.png}};
  \node at (.2,0){\includegraphics[scale=0.2]{mcmc.png}};
  \node at (.4,-0.2){\includegraphics[scale=0.2]{mcmc.png}};
  \node[anchor=north east] at (2.25,-1.7){MC of two different spatial frequencies $\nuis$};
  
 \end{scope}
 \end{scope}
 
\end{tikzpicture}
\end{center}
\caption{Graphical representation of the covariance $\ga$ (left) ---note the cone-like shape of the envelopes-- and an example of synthesized dynamics for narrow-band and broad-band Motion Clouds (right).}
\label{fig:rpt3D}
\end{figure}
%-----------------------------------------------------------
Plugging these expressions~\eqref{eq:fz-fth} into the definition~\eqref{eq-dfn-mc-spectrum} of the power spectrum of the motion cloud, one obtains a parameterization which is very similar to the one originally introduced in~\citep{Simoncini12}.  The following table gives the speed $v_0$ and frequency $(\th_0,\z_0)$  central parameters in terms of amplitude and orientation, each one being coupled with the relevant dispersion parameters.
Figure~\ref{fig:warps} and~\ref{fig:rpt3D} shows a graphical display of the influence of these parameters. 
%\todo{Jonathan: Isn't it figure 1 which shows a graphical display of parameter influence \ldots}
\begin{center}
	\begin{tabular}{|c|c|c|c|}
		\hline
		& Speed  & Freq. orient. & Freq. amplitude \\\hline
		(mean, dispersion) &
			$(\vz, \sr)$ & 
			$(\th_0, \stheta)$ & 
			$(\z_0, \sz)$ or $(m_{\Z},d_{\Z}) $\\\hline
	\end{tabular}
\end{center}

\begin{rem}
Note that the final envelope of $\hat\ga$ is in agreement with the formulation that is used in~\citep{Leon12}. However, that previous derivation was based on a heuristic which intuitively emerged from a long interaction between modelers and psychophysicists. Herein, we justified these different points from first principles. 
%However, this latter derivation was based on a heuristic following a trial-and-error strategy between modelers and psychophysicists. Herein, we justified these different points in a constructive manner. 
\end{rem}
%: >>>> LAURENT IS HERE <<<<
%%%%%%
\NipsArxiv{
\begin{rem}
\label{sec-spde-model}
The MC model can equally be described as a stationary solution of a stochastic partial differential equation (sPDE). This sPDE formulation is important since we aim to deal with dynamic stimulation, which should be described by a causal equation which is local in time. 
This is crucial for numerical simulations, since, this allows us to perform real-time synthesis of stimuli using an auto-regressive time discretization. This is a significant departure from previous Fourier-based implementation of dynamic stimulation~\citep{Leon12,Simoncini12}. 
This is also important to simplify the application of MC inside a bayesian model of psychophysical experiments (see Section~\ref{sec-psychophysics})%, and in particular~\eqref{eq-mle-mc}). 
The derivation of an equivalent sPDE model exploits a spectral formulation of MCs as Gaussian Random fields. The full proof along with the synthesis algorithm can be found in \NipsArxiv{the supplementary material}{Section~\ref{sec-spde-model}}.
\end{rem}
}
{
\subsection{sPDE Formulation and Numerical Synthesis Algorithm}
\label{sec-spde-model}
% The formulation of the MC gives an explicit parameterization of the covariance over the Fourier domain. We show here that it can be equivalently approximated by the solutions of a local PDE driven by a Gaussian noise. This formulation is important since we aim to deal with dynamic stimulation, which should be described by a causal equation which is local in time. This is indeed crucial to offer a fast simulation algorithm and to offer a coherent Bayesian inference framework, as shown in the following sections.
The MC model can equally be described as a stationary solution of a stochastic partial differential equation (sPDE). This sPDE formulation is important since we aim to deal with dynamic stimulation, which should be described by a causal equation which is local in time. 
This is crucial for numerical simulations, since, this allows us to perform real-time synthesis of stimuli using an auto-regressive time discretization. This is a significant departure from previous Fourier-based implementation of dynamic stimulation~\citep{Leon12,Simoncini12}. 
This is also important to simplify the application of MC inside a bayesian model of psychophysical experiments (see Section~\ref{sec-psychophysics})%, and in particular~\eqref{eq-mle-mc}). 
The derivation of an equivalent sPDE model exploits a spectral formulation of MCs as Gaussian Random fields. The full proof along with the synthesis algorithm can be found in \NipsArxiv{the supplementary material}{Section~\ref{sec-spde-model}}.
} 
% !TEX root = ../MotionClouds-NIPS.tex
\section{Psychophysical Study: Speed Discrimination}
\label{sec-psychophysics}

To exploit the useful features of our MC model and provide a generalizable proof of concept based on motion perception, we consider here the problem of judging the relative speed of moving dynamical textures and the impact  of both average spatial frequency and average duration of temporal correlations.

%%%%%%%%%%%%%%%%%%%%%%%%%%%%%%%%%%%%%%%%%%%%%
\subsection{Methods}

The task was to discriminate the speed $\q \in \RR$ of MC stimuli moving with a horizontal central speed $\vz = (\q,0)$. We assign as independent experimental variable the most represented spatial frequency $m_{\Z}$, that we denote in the following $\nuis$ for easier reading. The other parameters are set to the following values
\NipsArxiv{ $\sr = \frac{1}{t^\star \z_0}, \quad
	\th_0 = \frac \pi 2, \quad
	\stheta = \frac{\pi}{12}, \qandq
	d_\Z = 1.0~\si{c/\degree}.$
	}
	{
\eq{
	\sr = \frac{1}{t^\star \z_0}, \quad
	\th_0 = \frac \pi 2, \quad
	\stheta = \frac{\pi}{12}, \quad
	d_\Z = 1.0~\si{c/\degree} .
 }
	}
Note that $\sr$ is thus dependent of the value of $\z_0$ (that is computed from $m_{\Z}$ and $d_{\Z}$, see \NipsArxiv{Remark~\ref{rem:param} and the supplementary}{Remark~\ref{rem:param} and Section~\ref{sec-param-fz}}~) to ensure that $t^\star = \frac{1}{\sr\z_0}$ stays constant. This parameter $t^\star$ controls the temporal frequency bandwidth, as illustrated on the middle of Figure~\ref{fig:rpt3D}. 
We used a two alternative forced choice (2AFC) paradigm.
In each trial a grey fixation screen with a small dark fixation spot was followed by two stimulus intervals of $250~\ms$ each, separated by a grey 250~\ms\ inter-stimulus interval. The first stimulus had parameters $(\q_1,\nuis_1)$ and the second had parameters $(\q_2,\nuis_2)$. 
At the end of the trial, a grey screen appeared asking the participant to report which one of the two intervals was perceived as moving faster by pressing one of two buttons, that is whether $\q_1>\q_2$ or $\q_2>\q_1$.

Given reference values $(\q^\star,\nuis^\star)$, for each trial, $(\q_1,\nuis_1)$ and $(\q_2,\nuis_2)$ are selected so that 
\eq{
	\choice{
		\q_i = \q^\star, \;
		\nuis_i \in \nuis^\star + \Delta_\Nuis \\
		\q_j \in \q^\star + \Delta_\Q, \;
		\nuis_j = \nuis^\star
	}
	\qwhereq
	\choice{
		\Delta_\Q = \{ -2, -1, 0, 1, 2 \},  \\
		\Delta_\Nuis = \{ -0.48, -0.21, 0, 0.32, 0.85 \}, 	
	}
  }
where $(i,j)=(1,2)$ or $(i,j)=(2,1)$ (i.e. the ordering is randomized across trials), and where $\nuis$ values are expressed in cycles per degree (\si{c/\degree}) and $\q$ values in \si{\degree/\second}. 
Ten repetitions of each of the 25 possible combinations of these parameters are made per block of 250 trials and at least four such blocks were collected per condition tested. 
The outcome of these experiments are summarized by psychometric curves $\hat\phi_{\q^\star,\nuis^\star}$, where for all $(\q-\q^\star,\nuis-\nuis^\star) \in \Delta_\Q \times \Delta_\Nuis$, the value $\hat\phi_{\q^\star,\nuis^\star}(\q,\nuis)$ is the empirical probability (each averaged over the typically 40 trials) that a stimulus generated with parameters $(\q^\star, \nuis)$ is moving faster than a stimulus with parameters $(\q,\nuis^\star)$.

To assess the validity of our model, we tested four different scenarios by considering all possible choices among
\NipsArxiv{$\nuis^\star = 1.28~\si{c/\degree}, 
	\quad
	\q^\star \in \{5\si{\degree/\second}, 10\si{\degree/\second}\}, 
	\qandq
	t^\star \in \{0.1s, 0.2s\}$,}
{
\eq{
	\nuis^\star = 1.28~\si{c/\degree}, 
	\quad
	\q^\star \in \{5\si{\degree/\second}, 10\si{\degree/\second}\}, 
	\quad
	t^\star \in \{0.1s, 0.2s\},
  }
}
which corresponds to combinations of low/high speeds and a pair of temporal frequency parameters. 
Stimuli were generated on a Mac running OS 10.6.8 and displayed on a 20" Viewsonic p227f monitor with resolution $1024\times 768$ at 100~\si{\Hz}. Routines were written using Matlab 7.10.0 and Psychtoolbox 3.0.9 controlled the stimulus display. Observers sat 57~\si{\cm} from the screen in a dark room. Three observers with normal or corrected to normal vision took part in these experiments. They gave their informed consent and the experiments received ethical approval from the Aix-Marseille Ethics Committee in accordance with the declaration of Helsinki.

%%%
\subsection{Bayesian modeling}

To make full use of our MC paradigm in analyzing the obtained results, we follow the methodology of the Bayesian observer used for instance in~\citep{Stocker06,SotiropoulosVR,jogan2015signal}. We assume the observer makes its decision using a Maximum A Posteriori (MAP) estimator
\NipsArxiv{$\label{eq-map} \hat \q_{\nuis}(\m) = \uargmin{\q} [ -\log(\PP_{\M|\Q,\Nuis}(\m|\q,\nuis)) - \log(\PP_{\Q|\Nuis}(\q|\nuis)) ]$}
{
\eql{\label{eq-map}
	\hat \q_{\nuis}(\m) = \uargmin{\q} [ -\log(\PP_{\M|\Q,\Nuis}(\m|\q,\nuis)) - \log(\PP_{\Q|\Nuis}(\q|\nuis)) ]
	}
}
computed from some internal representation $\m \in \RR$ of the observed stimulus. For simplicity, we assume that the observer estimates $\nuis$ from $\m$ without bias.
%\todo{Jonathan: below, what do you mean by exposition? perhaps it could be clearer \ldots}
To simplify the numerical analysis, we assume that the likelihood is Gaussian, with a variance independent of $\q$. Furthermore, we assume that the prior is Laplacian as this gives a good description of the a priori statistics of speeds in natural images~\citep{Dong10}: %\todo{another good model is the "half-gaussian"}
\eql{\label{eq-likelihood-prior}
	\PP_{\M|\Q,\Nuis}(\m|\q,\nuis) = \frac{1}{\sqrt{2\pi} \sigma_\nuis} e^{ -\frac{|\m-\q|^2}{2\sigma_\nuis^2} }
	\qandq
	\PP_{\Q|\Nuis}(\q|\nuis) \propto e^{a_\nuis \q} 1_{[0,\q_{\max}]}(\q).
}
where $\q_{\max}>0$ is a cutoff speed ensuring that $\PP_{\Q|\Nuis}$ is a well defined density even if $a_\nuis>0$. %\todo{Laurent thinks that ensuring $a_\nuis\leq 0$ that is something like $a_\nuis = -1/\sigma_\nuis^2$  would be a more "economic" definition}
Both $a_\nuis$ and $\sigma_\nuis$ are unknown parameters of the model, and are obtained from the outcome of the experiments by a fitting process we now explain.

%%%
\subsection{Likelihood and Prior Estimation}
 
Following for instance~\citep{Stocker06,SotiropoulosVR,jogan2015signal}, the theoretical psychophysical curve obtained by a Bayesian decision model is 
\eq{
	\phi_{\q^\star,\nuis^\star}( \q,\nuis ) \eqdef \EE( \hat\q_{\nuis^\star}(\M_{\q,\nuis^\star}) > \hat\q_{\nuis}(\M_{\q^\star,\nuis}) )
} 
where $\M_{\q,\nuis} \sim \Nn(\q, \sigma_\nuis^2)$ is a Gaussian variable having the distribution $\PP_{\M|\Q,\Nuis}(\cdot|\q,\nuis)$.

The following proposition shows that in our special case of Gaussian prior and Laplacian likelihood, it can be computed in closed form. Its proof follows closely the derivation of~\cite[Appendix A]{SotiropoulosVR}, and can be found in \NipsArxiv{the supplementary materials}{Section~\ref{sec-proof-bayesian}}. 

\begin{prop}
In the special case of the estimator~\eqref{eq-map} with a parameterization~\eqref{eq-likelihood-prior}, one has
\eql{\label{eq-psycurve-theo}
	\phi_{\q^\star,\nuis^\star}( \q,\nuis ) 
	= 
	\psi\pa{ 
		\frac{ 
			\q-\q^\star - a_{\nuis^\star} \sigma_{\nuis^\star}^2  + a_{\nuis} \sigma_{\nuis}^2 
		}{ 
			\sqrt{ \sigma_{\nuis^\star}^2 + \sigma_{\nuis}^2  } 
		}
	}
}
where $\psi(t)=\frac{1}{\sqrt{2\pi}}\int_{-\infty}^t e^{-s^2/2} \d s$ is a sigmoid function. 
\end{prop}

One can fit the experimental psychometric function to compute the perceptual bias term $\mu_{\nuis,\nuis^\star} \in \RR$ and an uncertainty $\la_{\nuis,\nuis^\star}$ such that
\NipsArxiv{\label{eq-psycurve-xp}$
	\hat \phi_{\q^\star,\nuis^\star}( \q,\nuis ) 
	\approx 
	\psi\pa{ 
		\frac{ 
			\q-\q^\star - \mu_{\nuis,\nuis^\star}
		}{ 
			\la_{\nuis,\nuis^\star} 
		}
	}.$}
{
\eql{\label{eq-psycurve-xp}
	\hat \phi_{\q^\star,\nuis^\star}( \q,\nuis ) 
	\approx 
	\psi\pa{ 
		\frac{ 
			\q-\q^\star - \mu_{\nuis,\nuis^\star}
		}{ 
			\la_{\nuis,\nuis^\star} 
		}
	}.
 }
}
\begin{rem} Note that in practice we perform a fit in a log-speed domain \textit{ie} we consider $\phi_{\tilde \q^\star,\nuis^\star}( \tilde\q,\nuis )$ where $\tilde \q = \ln(1 + \q / \q_0) $ with $\q_0 = 0.3 \si{\degree/\second} $ following \citep{Stocker06}.
\end{rem}
By comparing the theoretical and experimental psychopysical curves~\eqref{eq-psycurve-theo} and~\eqref{eq-psycurve-xp}, one thus obtains the following expressions
\NipsArxiv{$ \sigma_{\nuis}^2 = \la_{\nuis,\nuis^\star}^2 - \frac{1}{2}\la_{\nuis^\star,\nuis^\star}^2 % \sigma_{\nuis^\star}^2   
	\qandq
	a_{\nuis}  = 
	a_{\nuis^\star} \frac{\sigma_{\nuis^\star}^2}{\sigma_{\nuis}^2} 
	- 
	\frac{\mu_{\nuis,\nuis^\star}}{\sigma_{\nuis}^2}.$}
{
\eq{
	\sigma_{\nuis}^2 = \la_{\nuis,\nuis^\star}^2 - \frac{1}{2}\la_{\nuis^\star,\nuis^\star}^2 % \sigma_{\nuis^\star}^2   
	\qandq
	a_{\nuis}  = 
	a_{\nuis^\star} \frac{\sigma_{\nuis^\star}^2}{\sigma_{\nuis}^2} 
	- 
	\frac{\mu_{\nuis,\nuis^\star}}{\sigma_{\nuis}^2}.
	% a_{\nuis^\star} \sigma_{\nuis^\star}^2  - a_{\nuis} \sigma_{\nuis}^2 = \mu_{\nuis,\nuis^\star}.
 }
}
The only remaining unknown is $a_{\nuis^\star}$, that can be set as any negative number based on previous work on low speed priors or, alternatively estimated in future by performing a wiser fitting method.
%\todo{Jonathan: is the above statement still true about a new central spatial frequency allowing the prior to be determined? \ldots}

%-------------------------------------------------------------%
%: fig:psycho
\begin{figure}[!ht]
\begin{center}
\subfigure{\bf (a) \label{fig:psychoa}}\includegraphics[scale=0.24]{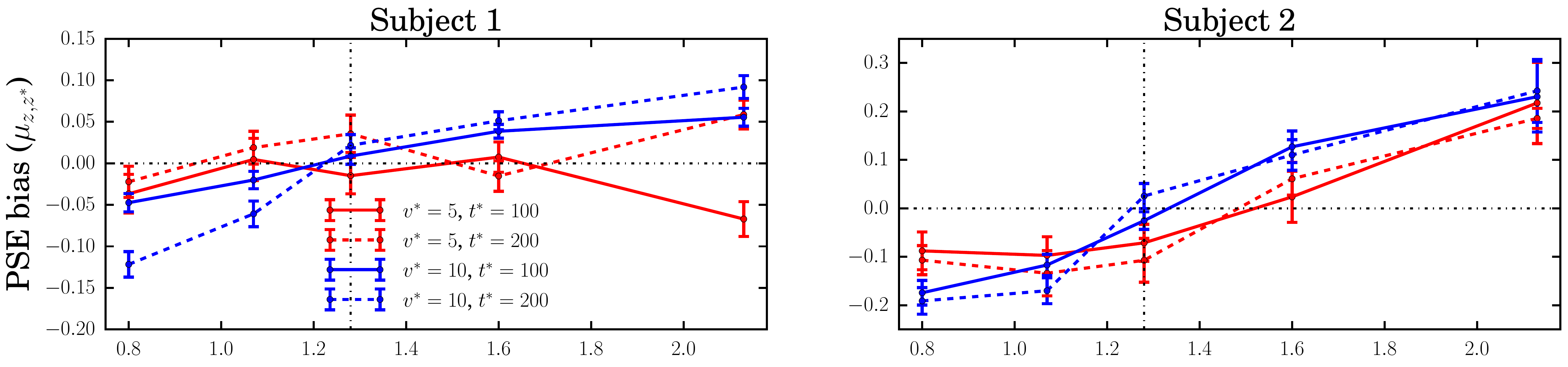}
\subfigure{\bf (b) \label{fig:psychob}}\includegraphics[scale=0.24]{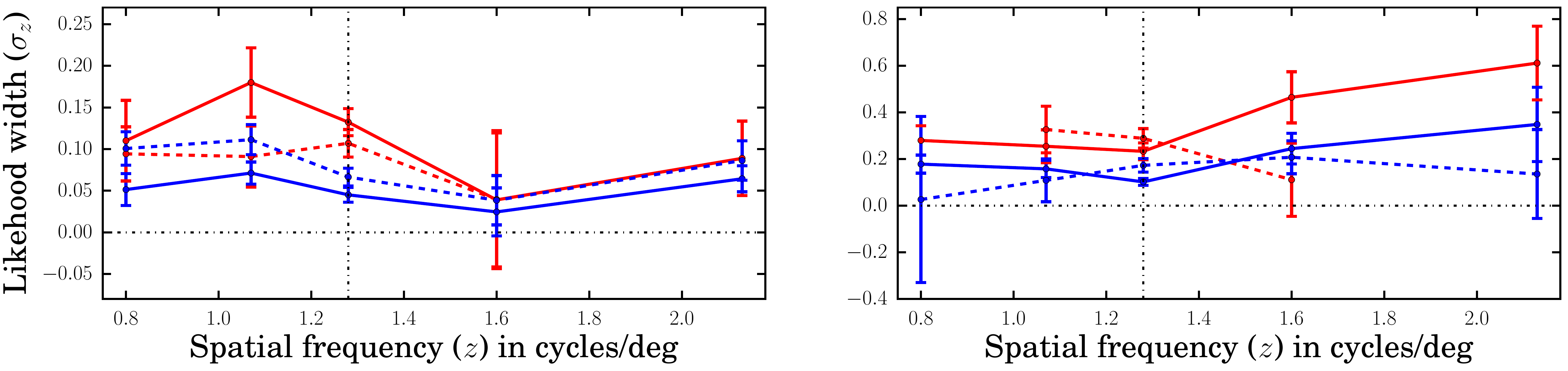}
\end{center}
\caption{\emph{2AFC speed discrimination results.} \textbf{(a)} Task generates psychometric functions which show shifts in the point of subjective equality for the range of test $\nuis$. Stimuli of lower frequency with respect to the reference (intersection of dotted horizontal and vertical lines gives the reference stimulus) are perceived as going slower, those with greater mean frequency are perceived as going relatively faster. This effect is observed under all conditions but is stronger at the highest speed and for subject 1.
\textbf{(b)} The estimated $\sigma_{\nuis}$ appear noisy but roughly constant as a function of $\nuis$ for each subject. Widths are generally higher for $v=5$ (red) than $v=10$ (blue) traces. The parameter $t^\star$ does not show a significant effect across the conditions tested.}
\label{fig:psycho}
\end{figure}
%-------------------------------------------------------------%
%%%
\subsection{Psychophysic Results}

The main results are summarized in Figure~\ref{fig:psycho} showing the parameters $\mu_{\nuis,\nuis^\star}$ in Figure~\ref{fig:psychoa} and the parameters $\sigma_\nuis$ in Figure~\ref{fig:psychob}. %The main conclusion is %\todo{Gab: why ``both"?} \done{positive effect + no due to likelihood width} 
Spatial frequency has a positive effect on perceived speed; speed is systematically perceived as faster as spatial frequency is increased, moreover this shift cannot simply be explained to be the result of an increase in the likelihood width (Figure~\ref{fig:psychob}) at the tested spatial frequency, as previously observed for contrast changes \cite{Stocker06,SotiropoulosVR}. Therefore the positive effect could be explained by a negative effect in prior slopes $a_\nuis$ as the spatial frequency increases. However, we do not have any explanation for the observed constant likelihood width as it is not consistent with the speed width of the stimuli $\sr = \frac{1}{t^\star \z_0}$ which is decreasing with spatial frequency. 

\subsection{Discussion}

We exploited the principled and ecologically motivated parameterization of MC to ask about the effect of scene scaling on speed judgements. In the experimental task, MC stimuli, in which the spatial scale content was systematically varied (via frequency manipulations) around a central frequency of 1.28~\si{c/\degree} were found to be perceived as slightly faster at higher frequencies slightly slower at lower frequencies. The effects were most prominent at the faster speed tested, of 10~\si{\degree/s} relative to those at 5~\si{\degree/s}. The fitted psychometic functions were compared to those predicted by a Bayesian model in which the likelihood or the observer's sensory representation was characterised by a simple Gaussian. Indeed, for this small data set intended as a proof of concept, the model was able to explain these systematic biases for spatial frequency as shifts in our \textit{a priori} on speed during the perceptual judgements as the likelihood width are constant across tested frequencies but lower at the higher of the tested speeds. Thus having a larger measured bias given the case of the smaller likelihood width (faster speed) is consistent with a key role for the prior in the observed perceptual bias. 
   
% I have added a couple of lines of intepretation here based on the likelihood widths
A larger data set, including more standard spatial frequencies and the use of more observers, is needed to disambiguate the models predicted prior function.

% !TEX root = ../MotionClouds-NIPS.tex

\section{Conclusions}

We have proposed and detailed a generative model for the estimation of the motion of images based on a 
formalization of small perturbations from the observer's point of view during parameterized rotations, zooms and translations. We connected these transformations to descriptions of ecologically motivated movements of both observers and the dynamic world.  
%
%These degree of freedom correspond respectively to linear transformations in the log-polar representation that is often used for describing retinotopy and that they naturally map to simple remappings in retinotopic space.
%
The fast synthesis of naturalistic textures 
optimized to probe motion perception was then demonstrated, through fast GPU implementations applying auto-regression techniques with much potential for future experimentation.  This extends previous work from~\citep{Leon12} by providing an axiomatic formulation. Finally, we used the stimuli in a psychophysical task and showed that these textures allow one to further understand the processes underlying speed estimation. By linking them directly to the standard Bayesian formalism, we show that the sensory representations of the stimulus (the likelihoods) in such models can be described directly from the generative MC model. In our case we showed this through the influence of spatial frequency on speed estimation. We have thus provided just one example of how the optimised motion stimulus and accompanying theoretical work might serve to improve our understanding of inference behind perception.
% We will also consider an alternative model where direction is uniformly random but speed is log-normal. 

\section*{Acknowledgements}

\Acknowledgments{}
%
%% APPENDIX / SUPPLEMENTARY %%
\NipsArxiv{}{
\appendix
%!TEX root = ../MotionClouds-NIPS.tex

\section{Graphical Display of MC}
We recall that MC are stationary Gaussian random field with a parameterized power spectrum having the form
\eql{\label{eq-dfn-mc-spectrum-bis}
		\foralls (\xi,\tau) \in \RR^3, \; 
		\hat \ga(\xi,\tau) = \frac{ \fz\pa{  \norm{\xi} } }{\norm{\xi}^2}			
			\ftheta\pa{ \angle{\xi} }
		    \Ll(\fr)\pa{ \vzMod\cos(\vzAng- \angle{\xi} ) - \frac{\tau}{  \norm{\xi} } }.
}
Similarly as was previously proposed in~\citep{Leon12}.
We show in Figure~\ref{wMC} two examples of such stimuli for different spatial frequency bandwidths. In particular, by tuning this bandwidth we could dissociate their respective role in action and perception~\citep{Leon12,Simoncini12}. Extending the study of visual perception to other dimensions, such as orientation or speed bandwidths, should provide essential data to titrate their respective role in motion integration.

\begin{figure}[ht!]
\begin{tabular}{m{.5\textwidth} m{.5\textwidth}}
\textbf{A} & \textbf{B} \\
\includegraphics[width=.45\textwidth]{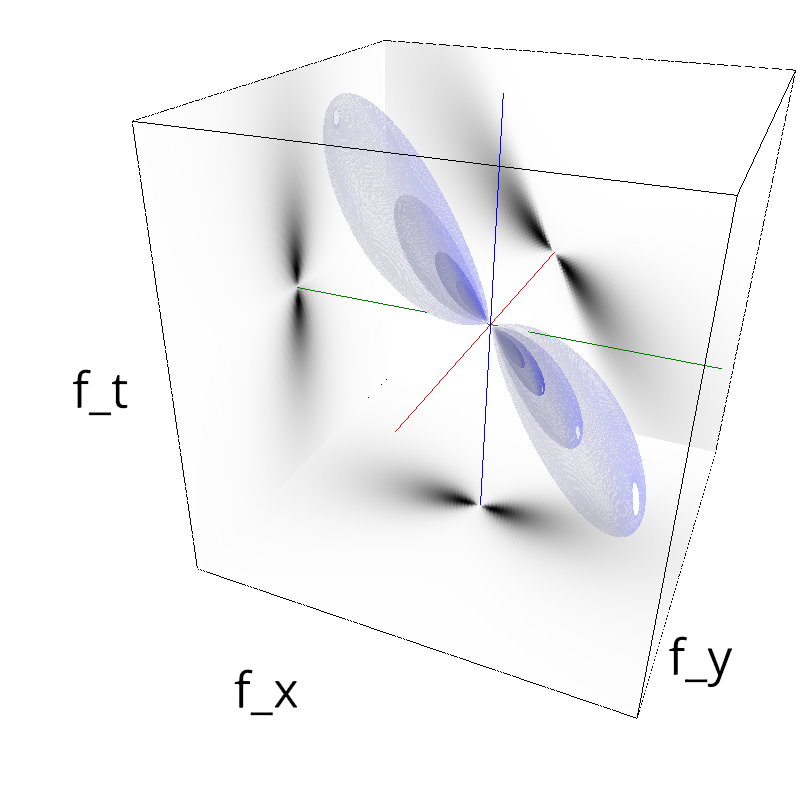}%
&
\includegraphics[width=.45\textwidth]{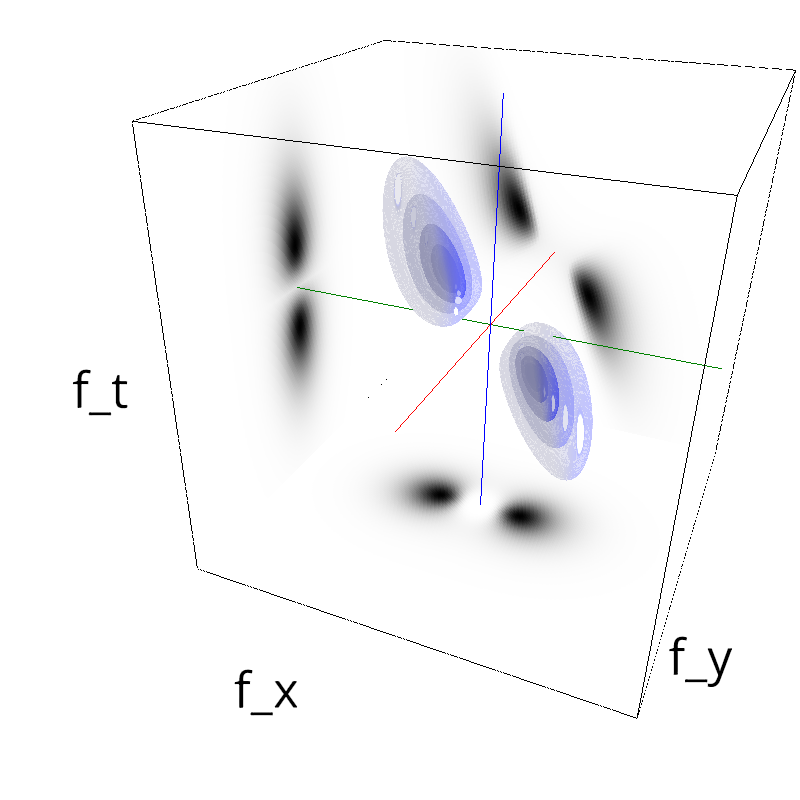}%
\\
\includegraphics[width=.45\textwidth]{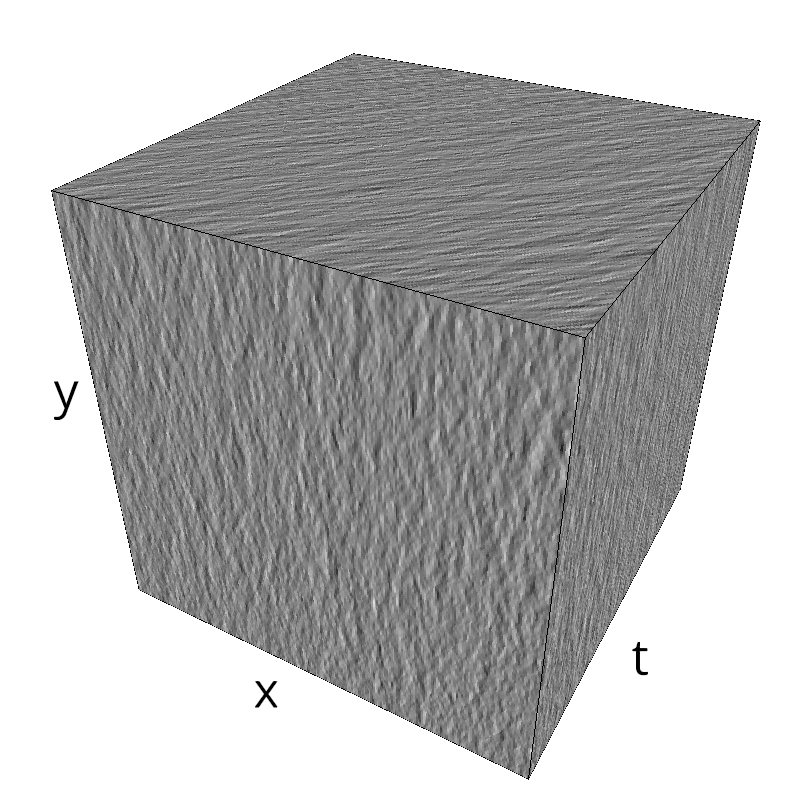}%
&%
\includegraphics[width=.45\textwidth]{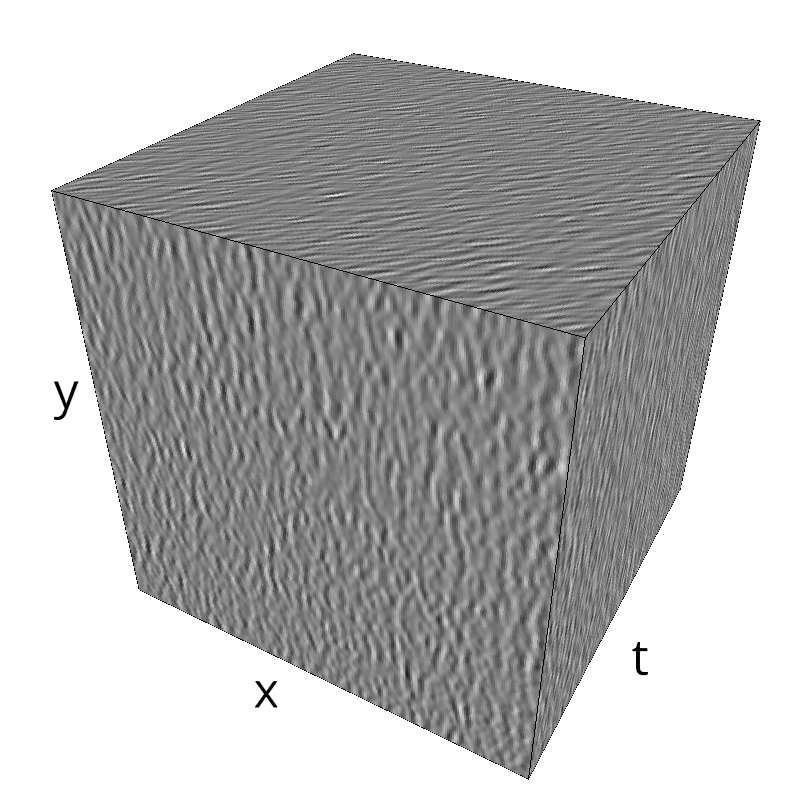}%
\end{tabular}
\caption{Broadband vs. narrowband stimuli. We show in (\textbf{A}) and (\textbf{B}) instances of the same Motion Clouds with different frequency bandwidths $\sz$, while all other parameters (such as $z_{0}$) are kept constant.
    The top column displays
    iso-surfaces of the spectral envelope by displaying enclosing volumes at 
    different energy values with respect to the peak amplitude of the Fourier spectrum.
    The bottom column shows an isometric view of the faces of the movie cube.
    The first frame of the movie lies on the x-y plane, the x-t plane lies on the
    top face and motion direction is seen as diagonal lines on this face (vertical
    motion is similarly see in the y-t face). 
The Motion Cloud with the broadest bandwidth is thought to best represent natural stimuli, since, as those, it contains many frequency components.
 (\textbf{A}) $\sz=0.25$, (\textbf{B}) $\sz=0.0625$.
}
\label{wMC}
\end{figure}
% !TEX root = ../MotionClouds-NIPS.tex

%%%%%%
\section{sPDE Formulation and Numerics}
\label{sec-spde-model-ap}

The formulation of the MC gives an explicit parameterization~\eqref{eq-dfn-mc-spectrum-bis} of the covariance over the Fourier domain. We show here that it can be equivalently discretized by the solutions of a local PDE driven by a Gaussian noise. This formulation is important since we aim to deal with dynamic stimulation, which should be described by a causal equation which is local in time. This is indeed crucial to offer a fast simulation algorithm (see Section~\ref{sec-ar2-numerics}) and to offer a coherent Bayesian inference framework, as shown in Section~\ref{sec-mle-estimator}.

%%%
\subsection{Dynamic Textures as Solutions of sPDE}

A MC $I$ with speed $\vz$ can be obtained from a MC $I_0$ with zero speed by the constant speed time warping
\eql{\label{eq-time-warping}
	I(x,t) \eqdef I_0(x-\vz t, t).
}
We now restrict our attention to $I_0$. 

We consider Gaussian random fields defined by a stochastic partial differential equation (sPDE) of the form
\eql{\label{eq-spde}
	\Dd(I_0) = \pd{W}{t}(x)
	\qwhereq
	\Dd(I_0) \eqdef \pdd{I_0}{t}(x) + \al \star \pd{I_0}{t}(x) + \be \star I_0(x)
}
This equation should be satisfied for all $(x,t)$, and we look for Gaussian fields that are stationary solutions of this equation. 
In this sPDE, the driving noise $\pd{W}{t}$ is white in time (i.e. corresponding to the temporal derivative of a Brownian motion in time) and has a 2-D covariance $\Si_W$�in space and $\star$ is the spatial convolution operator. The parameters $(\al,\be)$ are 2-D spatial filters that aim at enforcing an additional correlation in time of the model. 
Section~\ref{sec-spde} explains how to choose $(\al,\be,\Si_W)$ so that the stationary solutions of~\eqref{eq-spde} have the power spectrum given in~\eqref{eq-dfn-mc-spectrum-bis} (in the case that $\vz=0$), i.e. are motion clouds.

This sPDE formulation is important since we aim to deal with dynamic stimulation, which should be described by a causal equation which is local in time. This is crucial for numerical simulation (as explained in Section~\ref{sec-ar2-numerics}) but also to simplify the application of MC inside a bayesian model of psychophysical experiments (see Section~\ref{sec-mle-estimator}).

While it is beyond the scope of this paper to study theoretically this equation, one can show existence and uniqueness results of stationary solutions for this class of sPDE under stability conditions on the filers $(\al,\be)$ (see for instance~\cite{UnserBook}) that we found numerically to be always satisfied in our simulations.  
Note also that one can show that in fact the stationary solutions to~\eqref{eq-spde} all share the same law. These solutions can be obtained by  solving the sODE~\eqref{eq-spde-freq} forward for time $t>t_0$  with arbitrary boundary conditions at time $t=t_0$, and letting $t_0 \rightarrow -\infty$. This is consistent with the numerical scheme detailed in Section~\ref{sec-ar2-numerics}.

%%%
\subsection{Equivalence Between Spectral and sPDE MC Formulations}
\label{sec-spde}

The sPDE equation~\eqref{eq-spde} corresponds to a set of independent stochastic ODEs over the spatial Fourier domain, which reads, for each frequency $\xi$, 
\eql{\label{eq-spde-freq}
	\foralls t \in \RR, \quad
	\pdd{\hat I_0(\xi,t)}{t} + 
	\hat\al(\xi) \pd{\hat I_0(\xi,t)}{t} + 
	\hat\be(\xi) \hat I_0(\xi,t) = 
	\hsiW(\xi) \hat w(\xi,t)
}
where $\hat I_0(\xi,t)$ denotes the Fourier transform with respect to the space variable $x$ only. Here, $\hsiW(\xi)^2$ is the spatial power spectrum of $\pd{W}{t}$, which means that 
\eql{\label{eq-cov-rhs-spde}
	\Si_W(x,y)=c(x-y)
	\qwhereq
	\hat c(\xi) = \hsiW^2(\xi).
}
Here $\hat w(\xi,t) \sim \Nn(0,1)$ and $w$ is a  white noise in space and time. This formulation makes explicit that $(\hat\al(\xi),\hat\be(\xi))$ should be chosen in order to make the temporal covariance of the resulting process equal (or at least approximate) the temporal covariance appearing in~\eqref{eq-dfn-mc-spectrum-bis} in the motion-less setting (since we deal here with $I_0$), i.e. when $\vz=0$. 
This covariance should be localized around 0 and non-oscillating. It thus makes sense to constrain $(\hat\al(\xi),\hat\be(\xi))$ for the corresponding ODE~\eqref{eq-spde-freq} to be critically damped, which corresponds to imposing the following relationship
\eq{
	\foralls \xi, \quad
	\hat \al(\xi)=\frac{2}{\hat \nu(\xi)} \qandq \hat \be(\xi)=\frac{1}{\hat \nu^2(\xi)}
}
for some relaxation step size $\hat\nu(\xi)$.  The model is thus solely parameterized by the noise variance $\hat \siW(\xi)$ and the characteristic time $\hat \nu(\xi)$.

The following proposition shows that the sPDE model~\eqref{eq-spde} and the motion cloud model~\eqref{eq-dfn-mc-spectrum-bis} are identical for an appropriate choice of function $\fr$. 

\begin{prop}\label{prop-spde-equiv}
	When considering 
	\eql{\label{eq-expression-fr} 
		\foralls r>0, \quad \fr(r) = \Ll^{-1}(h)(r/\sr) 
		\qwhereq
		h(u) = (1+u^2)^{-2} 
} 
where $\Ll$ is defined in~\eqref{eq-dfn-mc-spectrum-bis}, 
	equation~\eqref{eq-spde-freq} admits a solution $I$ which is a stationary Gaussian field with power spectrum~\eqref{eq-dfn-mc-spectrum-bis} when setting
	\eql{\label{eq-equiv-spde-mc}
		\hsiW^2(\xi) =  \frac{1}{\hat\nu(\xi) \norm{\xi}^2} \fz(\norm{\xi})
			\ftheta(\angle{\xi}), 
		\qandq
		\hat \nu(\xi) = 
		 \frac{1}{ \sr \norm{\xi} }.
	} 
\end{prop}

\begin{proof}
	\newcommand{\MC}{\text{\tiny MC}}
	For this proof, we denote $I^\MC$ the motion cloud defined by~\eqref{eq-dfn-mc-spectrum-bis}, and $I$ a stationary solution of the sPDE defined by~\eqref{eq-spde}. We aim at showing that under the specification~\eqref{eq-equiv-spde-mc}, they have the same covariance. This is equivalent to show that $I_0^\MC(x,t) =  I^\MC(x+ct,t)$ has the same covariance as $I_0$. 
	One shows that for any fixed $\xi$, equation~\eqref{eq-spde-freq} admits a unique (in law) stationary solution $\hat I_0(\xi,\cdot)$ which is a stationary Gaussian process of zero mean and with a covariance which is $\hsiW^2(\xi)  r \star \bar r$ where $r$ is the impulse response (i.e. taking formally $a=\de$) of the ODE $r'' + 2r'/u + r''/u^2 = a$ where we denoted $u=\hat\nu(\xi)$. This impulse response is easily shown to be $r(t) = t e^{-t/u} \one_{\RR^+}(t)$. The covariance of $\hat I_0(\xi,\cdot)$ is thus, after some computation, equal to $\hsiW^2(\xi)  r \star \bar r = \hsiW^2(\xi) h(\cdot/u)$ where $h(t) \propto (1+|t|)e^{-|t|}$.
	Taking the Fourier transform of this equality, the power spectrum $\hat\ga_0$ of $I_0$ thus reads
	\eq{\label{eq-powerspec-spde}
		\hat \ga_0(\xi,\tau) = \hsiW^2(\xi) \hat\nu(\xi) h(\hat\nu(\xi) \tau)
		\qwhereq
		h(u) = \frac{1}{(1+u^2)^2}
	}
	and where it should be noted that this $h$ function is the same as the one introduced in~\eqref{eq-expression-fr}.
%	Note that we recover $\hat h=\Fs$, where $\Fs$ is the function introduced in~\eqref{eq:dist_speed}.
	%
	The covariance $\ga^\MC$ of $I^\MC$ and $\ga_0^\MC$ of $I_0^\MC$ are related by the relation 
	\eq{\label{eq-powerspec-MC}		
		\hat \ga_0^{\MC}(\xi,\tau) = \hat \ga^{\MC}(\xi,\tau-\dotp{\xi}{\vz}) = 
		\frac{ 1 }{\norm{\xi}^2} \fz( \norm{\xi} )
		\ftheta\pa{ \angle{\xi}  } 
		h\pa{-\frac{\tau }{  \sr \norm{\xi} } }.
	}
	where we used the expression~\eqref{eq-dfn-mc-spectrum-bis} for $\hat \ga^{\MC}$ and the value of $\Ll(\fr)$ given by~\eqref{eq-expression-fr}.
	Condition~\eqref{eq-equiv-spde-mc} guarantees that expression~\eqref{eq-powerspec-spde} and~\eqref{eq-powerspec-MC} coincide, and thus $\hat \ga_0=\hat \ga_0^{\MC}$.
\end{proof}

%%%%%%%%
\subsection{Expression for $\fr$}
\label{sec-expression-fr}

%\todo{I guess one needs to restrict integration to $[-\pi/2,\pi/2]$. Furthermore, I guess $\fr$ should be defined up to a constant, that can the be chosen to ensure normalization to unit mass. }

Equation~\eqref{eq-expression-fr} states that in order to obtain a perfect equivalence between the MC defined by~\eqref{eq-dfn-mc-spectrum-bis} and by~\eqref{eq-spde}, the function has $\Ll^{-1}(h)$ to be well-defined. It means we need to compute the inverse of the transform of the linear operator $\Ll$
\eq{\label{Ltransform}
	\forall u \in \RR,  \quad \Ll(f)(u) = 2\int_{0}^{\pi/2} f(-u/ \cos(\phi) ) \d \phi.  	
}
to the function $h$. The following proposition gives a closed-form expression for this function, and shows in particular that it is a function in $L^1(\RR)$, i.e. it has a finite integral, which can be normalized to unity to define a density distribution. 
Figure~\ref{fig:Ltransform} shows a graphical display.

\begin{prop}
One has
\eq{
	\Ll^{-1}(h)(u) = \frac{2-u^2}{\pi(1+u^2)^2}-\frac{u^2(u^2+4)(\log(u)-\log(\sqrt{u^2+1}+1))}{\pi(u^2+1)^{5/2}}.
}
In particular, one has
\eq{
	\Ll^{-1}(h)(0) = \frac{2}{\pi}
	\qandq
	\Ll^{-1}(h)(u) \sim \frac{1}{2\pi u^3}
	\quad\text{when}\quad
	u \rightarrow +\infty.
}
\end{prop}
\begin{proof}
 The variable substitution $x=\cos(\phi)$ allows to rewrite \eqref{Ltransform} as $$ \forall u \in \RR,  \quad \Ll(h)(u) = 2\int_{0}^{1} h\left(-\frac{u}{x} \right) \frac{x}{\sqrt{1-x^2}} \frac{\d x}{x}.$$ 
In such a form, we recognize a Mellin convolution which could be inverted by the use of Mellin convolution table.
	%\todo{Gab: Needs details}
\end{proof}

\begin{figure}[!ht]
\begin{center}
\includegraphics[scale=0.5]{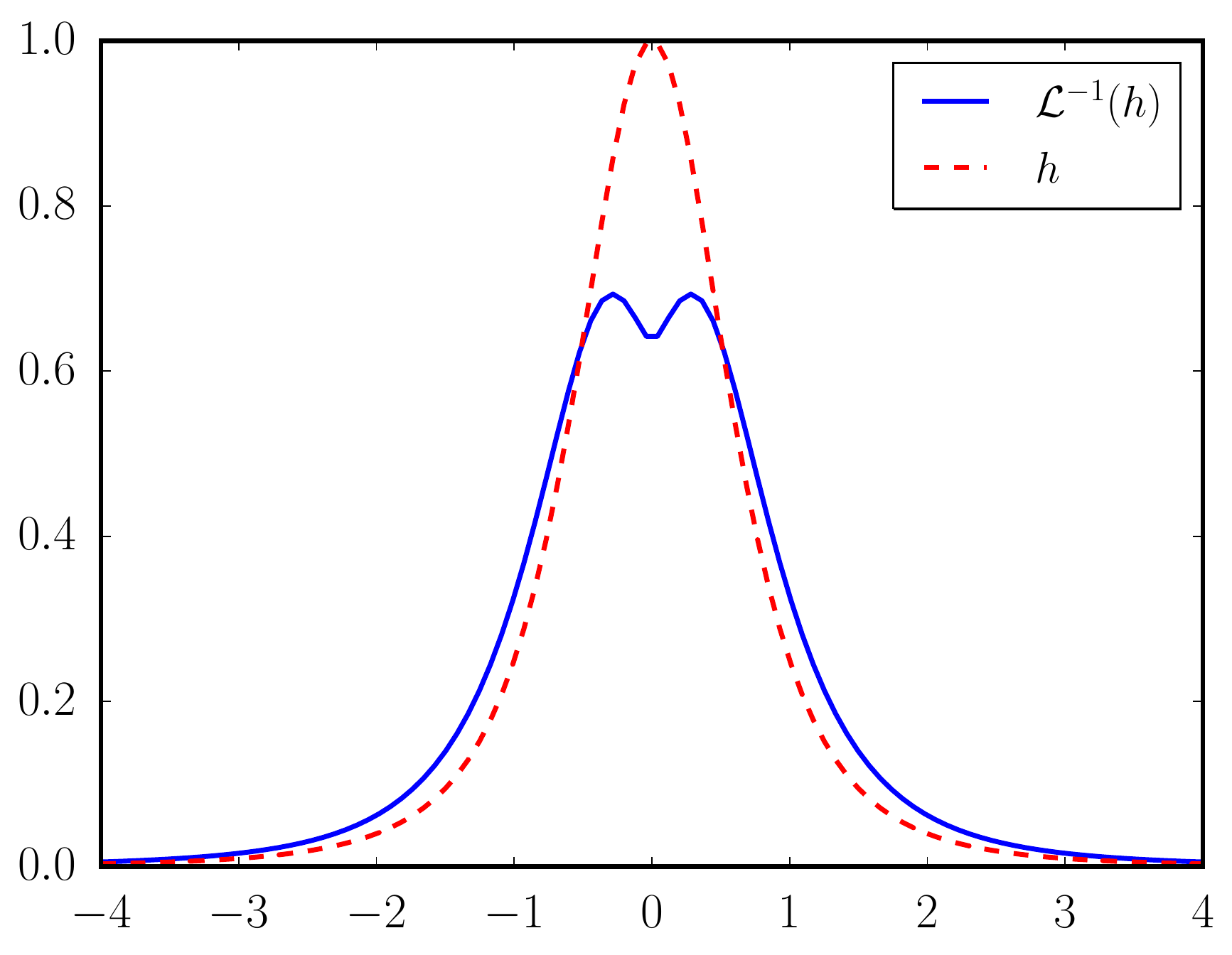}
\end{center}
\caption{Functions $h$ and $\Ll^{-1}(h)$.}%\todo{Use .eps files.}}
\label{fig:Ltransform}
\end{figure}

%%%
\subsection{Parametrization of $\fz$}
\label{sec-param-fz}

\paragraph{Parametrization by mode and standard deviation}
The log-normal distribution could be written 
\eq{
\fz(\z) \propto \dfrac{\z_0}{\z} 
	e^{
		-\frac{\ln\left( \frac{\z}{\z_0} \right)^2}{2\ln\left(1 + \sz^2 \right)}
	}.
}
 The parameters $(\z_0, \sz)$ are convenient to write the distribution but they do not reflect remarkable values of a log-normal random variable. Instead, we prefer to fix directly the mode $ m_{\Z} = \argmax_{\z} \fz(\z) $ and standard deviation $d_{\Z} = \sqrt{ \int_{\RR_+} \z^2 \fz(\z) \d \z }$. These couples of variable are linked by the following equations,
\eq{
m_{\Z} = \frac{\z_0}{1+\sz^2} \qandq d_{\Z} = \z_0 \sz^2 (1+\sz^2).
} 
Such formula could be inverted by finding the unique positive root of 
\eq{
P(x) = x^2 (1+x^2)^2 - \frac{d_{\Z}}{m_{\Z}} 
}
because $P(\sz)$ = 0 and finally set $\z_0 = m_{\Z}(1+\sz^2)$.

\paragraph{Parametrization by mode and octave bandwidth}
Another choice would be to parametrize $\fz$ by its mode $m_{\Z}$ and octave bandwidth $B_{\Z}$ which is defined by 
\eq{
B_{\Z} = \frac{ \ln\left( \frac{z_+}{z_- }\right) }{ \ln(2) }
}
where $(z_-,z_+)$ are the half-power cutoff frequencies \textit{ie} verifies $ \fz(z_-) = \fz(z_+) = \frac{\fz(m_{\Z})}{2} $.
This last condition comes down to study the roots of the following polynomial 
\eq{
Q(X) = X^2 + 2 \ln(1+\sz^2) X - 2 \ln(2) \ln(1+\sz^2) + \frac{1}{2} \ln(1 + \sz^2)^2 
}
where $X = \ln \left( \frac{z}{z_0} \right)$.
It follows that 
\eq{
B_{\Z} = \sqrt{\frac{8\ln(1+\sz^2)}{\ln(2)}} .
}
Conversely,
\eq{
\sz = \sqrt{\exp\left(\frac{\ln(2)}{8} B_{\Z}^2 \right) - 1}.
}

%%%
\subsection{AR(2) Discretization of the sPDE}
\label{sec-ar2-numerics}

Most previous works (such as~\citep{Galerne11} for static and~\citep{Leon12,Simoncini12} for dynamic textures) have used global Fourier-based approach that makes use of the explicit power spectrum expression~\ref{eq-dfn-mc-spectrum-bis}. The main drawbacks of such an approach are: (i) it introduces an artificial periodicity in time and thus can only be used to synthesize a finite number of frames; (ii) the discrete computational grid may introduce artifacts, in particular when one of the bandwidths is of the order of the discretization step; (iii) these frames must be synthesized at once, before the stimulation, which prevents real-time synthesis.

To address these issues, we follow the previous works of~\citep{Doretto-Dyntex,2014-xia-siims}
and make use of an auto-regressive (AR) discretization of the sPDE~\eqref{eq-spde}. In contrast with these previous works, we use a second order AR(2) regression (in place of a first order AR(1) model). Using higher order recursions is crucial to be consistent with the continuous formulation~\eqref{eq-spde}. Indeed, numerical simulations show that AR(1) iterations lead to unacceptable temporal artifacts: in particular, the time correlation of AR(1) random fields typically decays too fast in time. % The detailed derivation of the AR(2) implementation of MC can be found in the supplementary materials. 

The discretization computes a (possibly infinite) discrete set of 2-D frames $(I_0^{(\ell)})_{\ell \geq \ell_0}$ separated by a time step $\De$, and we approach at time $t=\ell\De$ the derivatives as 
\eq{
	\pd{I_0(\cdot,t)}{t} \approx \De^{-1}( I_0^{(\ell)} - I_0^{(\ell-1)} )
	\qandq
	\pdd{I_0(\cdot,t)}{t} \approx \De^{-2}( I_0^{(\ell+1)} + I_0^{(\ell-1)} - 2 I_0^{(\ell)} ), 
}
which leads to the following explicit recursion 
\eql{\label{eq-ar-discr}
	\foralls \ell \geq \ell_0, \quad
	I_0^{(\ell+1)} = 
	( 2\de - \De \al - \De^2 \be) \star I_0^{(\ell)}  + 
	( -\de + \De \al ) \star I_0^{(\ell-1)} + 
	\De^2 W^{(\ell)}, 
}
where $\de$ is the 2-D Dirac distribution and where $(W^{(\ell)})_{\ell}$ are i.i.d. 2-D Gaussian field with distribution $\Nn(0,\Si_W)$, and $(I_0^{(\ell_0-1)}, I_0^{(\ell_0-1)})$ can be arbitrary initialized. 

One can show that when $\ell_0 \rightarrow -\infty$ (to allow for a long enough ``warmup'' phase to reach approximate time-stationarity) and $\De \rightarrow 0$, then $I_0^\De$ defined by interpolating $I_0^\De(\cdot,\De \ell) = I^{(\ell)}$ converges (in the sense of finite dimensional distributions) toward a solution $I_0$ of the sPDE~\eqref{eq-spde}. We refer to~\cite{Unser-Discrete} for a similar result in the 1-D case (stochastic ODE).
We implemented the recursion~\eqref{eq-ar-discr} by computing the 2-D convolutions with FFT's on a GPU, which allows us to generate high resolution videos in real time, without the need to explicitly store the synthesized video.

% !TEX root = ../MotionClouds-NIPS.tex

%%%%%%
\section{Experimental Likelihood vs. the MC Model}
\label{sec-mle-estimator}

In our paper, we propose to directly fit the likelihood $\PP_{\M|\Q,\Nuis}(\m|\q,\nuis)$ from the experimental psychophysical curve. While this makes sense from a data-analysis point of view, this required strong modeling hypothesis, in particular, that the likelihood is Gaussian with a variance $\sigma_\nuis^2$ independent of the parameter $\q$ to be estimated by the observer.

In this section, we direct a likelihood model directly from the stimuli, by making another (of course questionnable) hypothesis, that the observer uses a standard motion estimation process, based on the motion energy concept~\citep{Adelson85}, that we adapt here to the MC distribution. In this setting, this corresponds to using a MLE estimator, and making use of the sPDE formulation of MC. 

%%%%%
\subsection{MLE Speed Estimator}

We first show how to compute this MLE estimator. To be able to achieve this, the following proposition derive the sPDE satisfied by a motion cloud with a non-zero speed.

\begin{prop}
	A MC $I$ with speed $\vz$ can be defined as a stationary solution of the sPDE
	\eql{\label{eq-spde-warped}
		\Dd(I) + \dotp{\Gg(I)}{\vz} + \dotp{\Hh(I) \vz}{\vz} = \pd{W}{t}
	}
	where $\Dd$ is defined in~\eqref{eq-spde}, $\partial_x^2 I$ is the hessian of $I$ (second order spatial derivative), where 
	\begin{align*}
		\Gg(I) \eqdef \al \star \nabla_x I + 2  \partial_t \nabla_x I 
		\qandq
		\Hh(I) \eqdef (\partial_x^2 I)
	\end{align*}
	and $(\al,\be,\Si_W)$ are defined in Proposition~\ref{prop-spde-equiv}.
\end{prop}

\begin{proof}
	This follows by derivating in time the warping equation~\eqref{eq-time-warping}, denoting $y \eqdef x+\vz t$
	\begin{align*}
		\partial_t I_0(x,t) &= \partial_t I(y,t) + \dotp{\nabla I(y,t)}{\vz}, \\
		\partial_t^2 I_0(x,t) &= \partial_t^2 I(y,t) + 2\dotp{\partial_t\nabla I(y,t)}{\vz}
			+ \dotp{ \partial_x^2 I(y,t) \vz }{\vz} 
	\end{align*}
	and plugging this into~\eqref{eq-spde} after remarking that the distribution of $\pd{W}{t}(x,t)$ is the same as the distribution of $\pd{W}{t}(x-\vz t, t)$. 
\end{proof}

Equation~\eqref{eq-spde-warped} is useful from a Bayesian modeling perspective, because, informally, it can be interpreted as the fact that the Gaussian distribution of MC as the following appealing form, for any function $\Ii : \RR^2 \times \RR \rightarrow \RR$
\eq{
	\PP_I(\Ii) = \frac{1}{Z_I} \exp( -\norm{ 
		\Dd(\Ii) + \dotp{\Gg(\Ii)}{\vz} + \dotp{\Hh(\Ii) \vz}{\vz}
	 }_{\Si_W^{-1}}^2 )
}
where $Z_I$ is a normalization constant which is independent of $\vz$ and 
\eq{
	\norm{\Ii}_{\Si_W^{-1}}^2 \eqdef \dotp{\Ii}{\Ii}_{\Si_W^{-1}}
	\qandq
	\dotp{\Ii_1}{\Ii_2}_{\Si_W^{-1}}
	\eqdef \int \int \frac{\hat \Ii_1(\xi,t) \hat \Ii_2(\xi,t)^*}{ \hsiW^2(\xi) }  \d \xi \d t
}
where $\hsiW$ is defined in~\eqref{eq-cov-rhs-spde}.

This convenient formulation allows to re-write the MLE estimator of the horizontal speed $\q$ parameter of a MC as 
\eq{
	\hat \q^{\text{MLE}}(\Ii) \eqdef \uargmax{\q} \PP_{I}(\Ii)
	\qwhereq
	\vz = (\q,0) \in \RR^2
}
used to analyse psychophysical experiments as
\eql{\label{eq-MLE-fourthorder}
	\hat \q^{\text{MLE}}(\Ii) = \uargmin{\q} 
	\norm{ 
		\Dd(\Ii) + \q \dotp{\Gg(\Ii)}{(1,0)} + \q^2 \dotp{\Hh(\Ii) (1,0)}{(1,0)}
	 }_{\Si_W^{-1}}^2
}
where we used the fact that the normalizing constant $Z_I$ is independent of $\vz$. Expanding the squares shows that~\eqref{eq-MLE-fourthorder} is the optimization of a fourth order polynomial, whose solution can be computed in closed form as one of the roots of the derivative of this polynomial, which is hence a third order polynomial.

%%%
\subsection{MLE Modeling of the Likelihood}

In our paper, following several previous works such as~\citep{Stocker06,SotiropoulosVR}, we assumed the existence of an internal representation parameter $m$, which was assumed to be a scalar, with a Gaussian distribution conditioned on $(\q,\nuis)$.
We explore here the possibility that this internal representation could be directly obtained from the stimuli by the usage by the observer of an ``optimal'' speed detector (an MLE estimate). 

Denoting $I_{\q,\nuis}$ a MC, which is a random Gaussian field of power spectrum~\eqref{eq-dfn-mc-spectrum-bis}, with central speeds $\vz=(\q,0)$ and central spacial frequency $\nuis$ (the other parameters being fixed as explained in the experimental section of the paper), this means that we consider the internal representation as being the following scalar random variable 
\eql{\label{eq-mle-mc}
	\M_{\q,\nuis} \eqdef \hat \q_\nuis^{\text{MLE}}(I_{\q,\nuis})
	\qwhereq
	\hat \q_\nuis^{\text{MLE}}(\Ii) \eqdef \uargmax{\q} \PP_{\M|\Q,\Nuis}(\Ii|\q,\nuis), 
}
As detailed in~\eqref{eq-MLE-fourthorder} it can be efficiently computed numerically. 

As shown in Figure~\ref{fig:mlea}, we observed that $\M_{\q,\nuis}$ is well approximated by a Gaussian random variable. Its mean is nearly constant and very close to $\q$, and Figure~\ref{fig:mleb} shows the evolution of its variance. Our main finding is that this optimal estimation model (using an MLE) is not consistent with the experimental finding because the estimated standard deviations of observers don't show a decreasing behavior as in Figure~\ref{fig:mleb}.
\begin{figure}[!ht]
\begin{center}
\subfigure[Histogram  \label{fig:mlea}]{\includegraphics[scale=0.35]{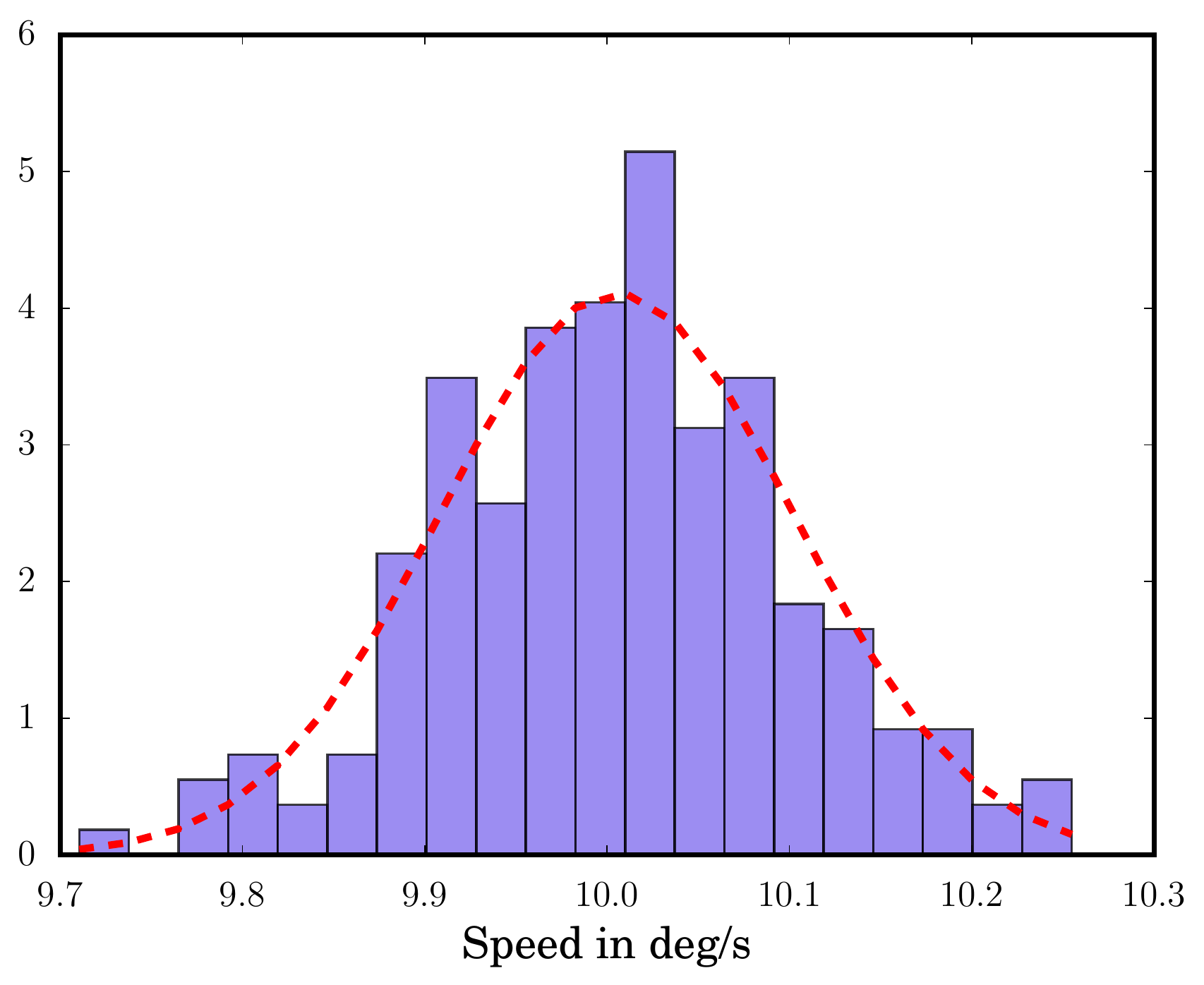}}
\subfigure[Standard deviation \label{fig:mleb}]{\includegraphics[scale=0.35]{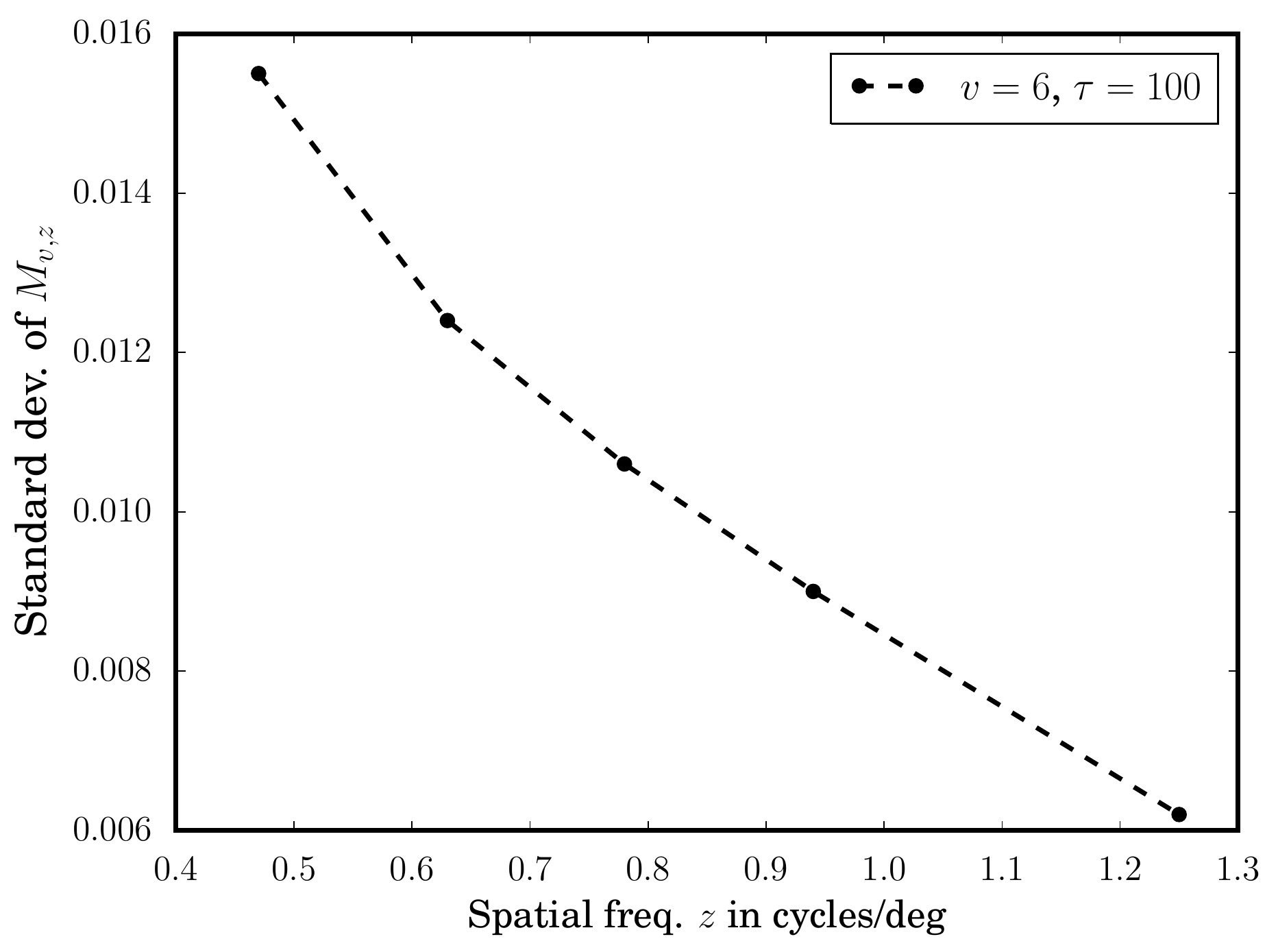}}
\end{center}
\caption{Estimates of $\M_{\q,\nuis}$ defined by \eqref{eq-mle-mc} and its standard deviation as a function of $z$.}
\label{fig:mle}
\end{figure}
% \todo{Gab: Use .eps files.}
\subsection{Prior slope and Likelihood width fitting}
In Section~3 we use equations 
\eq{
	\sigma_{\nuis}^2 = \la_{\nuis,\nuis^\star}^2 - \frac{1}{2}\la_{\nuis^\star,\nuis^\star}^2 % \sigma_{\nuis^\star}^2   
	\qandq
	a_{\nuis}  = 
	a_{\nuis^\star} \frac{\sigma_{\nuis^\star}^2}{\sigma_{\nuis}^2} 
	- 
	\frac{\mu_{\nuis,\nuis^\star}}{\sigma_{\nuis}^2}
	% a_{\nuis^\star} \sigma_{\nuis^\star}^2  - a_{\nuis} \sigma_{\nuis}^2 = \mu_{\nuis,\nuis^\star}.
}
to determine $a_{\nuis}$ and $\sigma_{\nuis}$. The slopes $a_{\nuis}$ are noisy due to the quotient  $\frac{\sigma_{\nuis^\star}^2}{\sigma_{\nuis}^2}$ therefore we only show some of the best fit in Figure \ref{fig:az} when the approximation $\sigma_{\nuis}^2$ constant holds. 
\begin{figure}[!ht]
\begin{center}
\includegraphics[scale=0.27]{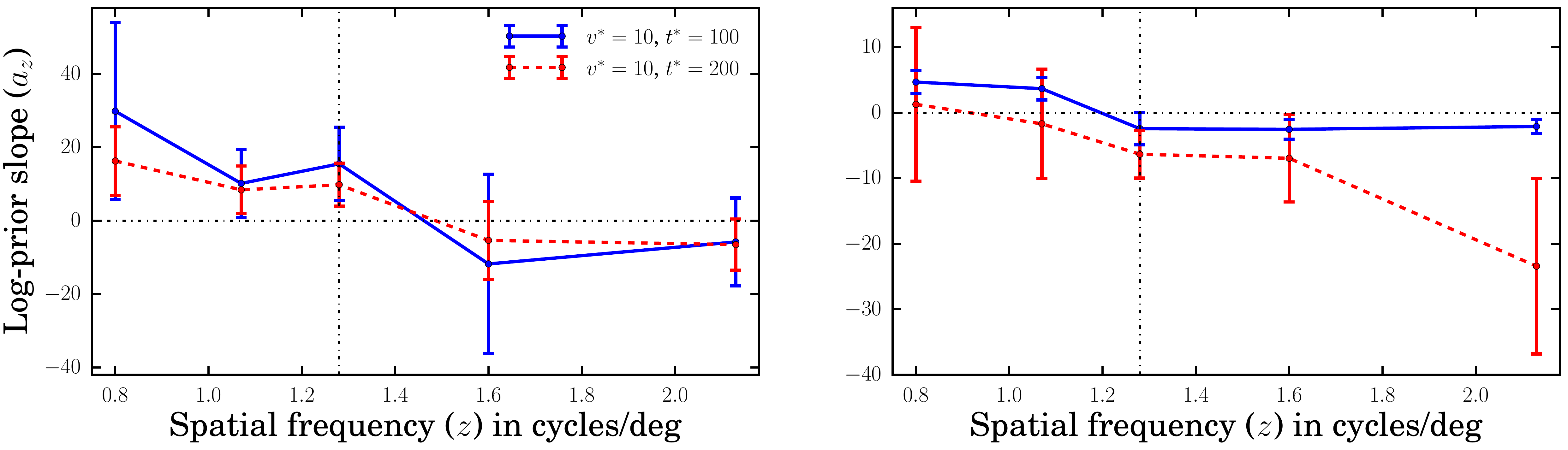}
\end{center}
\caption{Example of decreasing $a_z$. The unknown $a_{z^\star}$ choosen so that $\sum_z a_z^2$ is minimum.} %\todo{Gab: Use .eps files.}
\label{fig:az}
\end{figure}
\vspace{2mm}
%

%!TEX root = ../MotionClouds-NIPS.tex

\section{Proofs}

%%%%%%%%%%%%%%%%%%%%%%%%%%%%%%%%%%%%%%%%%%%%%%%%%%
\subsection{Proof of Proposition 2}
\label{sec-proof-prop-mc-spectrum}

We recall the expression of the covariance
\eql{\label{eq-cov-proposition-ap}
		\foralls (x,t) \in \RR^3, \quad
		\ga(x,t) = \int \int_{\RR^2} c_g(\phi_\geom(x-\speed t))  \fv(\speed) \falpha(\geom) \d \speed \d \geom
}

%%%%
We denote $(\th,\phi,\z,r) \in \Ga = [-\pi,\pi)^2 \times \RR_+^2$ the set of parameters. 
According to Proposition 1, %~\ref{eq-conv-deadleaves} 
the covariance of $I$ is $\gamma$ defined by \eqref{eq-cov-proposition-ap}. Denoting $h(x,t) = c_g( \z R_\th(x - \speed t ) )$, one has, in the sense of distributions (taking the Fourier transform with respect to $(x,t)$)
\eq{
	\hat h(\xi,\tau) = \z^{-2} \hat g( \z^{-1} R_{\th}(\xi) )^2 \de_{\Vv}(\speed)
	\qwhereq
	\Vv = \enscond{\speed  \in \RR^2 }{ \tau + \dotp{\xi}{\speed} = 0 }.
}
%
%\done{un $\z^{-2}$ apparait lorsque l'on fait le cgmt de variable dans la transformée de fourier, le power spectrum vient des zooms, le problème se pose alors du choix de la log-normale qui donnerait un $\z^{-3}$}
%
%\todo{Laurent: as the speed plane spans different speeds $v$ to $v+dv$, the corresponding shear covers a cone and it seems to me that its infinitesimal width should be proportional to $f_r$ - we thus needs a normalization by $\norm{\xi}$; meaning also that to get finally a power spectrum in $1/\norm{\xi}^2$, we do not need the square in the distribution for $\z$} where $\de_{\Vv}$ is the Dirac distribution along the surface $\Vv$ (the ``speed plane''). This thus implies that, when decomposing $v=\ga \vz \in \Span(\vz)$
%
Taking the Fourier transform of~\eqref{eq-cov-proposition-ap} and using this computation, one has
\begin{align*}
	\hat \ga(\xi,\tau) \!=\!\! 
	\int_\Ga \frac{1}{\z^2} 
	|\hat g\pa{ \z^{-1} R_{\th}(\xi) }|^2 
	\de_{\Vv}(\vz + r(\cos(\phi),\sin(\phi))) \ftheta(\th) \fz(\z) \fr(r)  
	\: \d \th \, \d \z \, \d r \, \d \phi .
\end{align*}
In the special case of $g$ being a grating, i.e. $|\hat g|^2 = \de_{\xiz}$, one has in the sense of distributions
\eq{
	\z^{-2} |\hat g\pa{ \z^{-1} R_{\th}(\xi) }|^2 = \de_{\Bb}(\th,\z)
	\qwhereq
	\Bb = \enscond{(\th,\z) }{ \z^{-1} R_{\th}(\xi) = \xiz }.
}
Observing that
$\de_{\Vv}(\speed) \de_{\Bb}(\th,\z) = \de_{\Cc}(\th,\z,r)$ where
\eq{
	\Cc = \enscond{(\th,\z,r)}{
		\z = \norm{\xi}, \; 
		\th =  \angle{\xi}, \; 
		r = -\frac{\tau}{ \norm{\xi}\cos(\angle{\xi}-\phi) } - \frac{\vzMod \cos(\angle{\xi} - \vzAng)}{\cos(\angle{\xi}-\phi)}
	}
}
one obtains the desired formula.
%one obtains
%\eq{
%	\hat \ga(\xi,\tau) = \frac{ 1 }{\norm{\xi}^2}
%	\fz( \norm{\xi} )
%	\ftheta\pa{ \angle{\xi} } 
%	\Fs\pa{\frac{\tau}{  \norm{\xi} } + \vzMod\cos(\angle{\xi}-\vzAng ) }.
%}
% where $ \Fs(u) = \int_{-\pi}^{\pi} \fr \left( -\frac{u}{ \cos(\phi-\angle{\xi}) } \right)  \d \phi $. 

%%%%%%%%%%%%%%%%%%%%%%%%%%%%%%%%%%%%%%%%%%
\subsection{Proof of Proposition~3}
\label{sec-proof-bayesian}

One has the closed form expression for the MAP estimator
\eq{
	\hat \q_\nuis(\m) = \m - a_\nuis \sigma^2_\nuis, 
}
and hence, denoting $\Nn(\mu,\sigma^2)$ the Gaussian distribution of mean $\mu$ and variance $\sigma^2$, 
\eq{
	\hat\q_{\nuis}(\M_{\q,\nuis}) \sim \Nn(\q-a_\nuis \sigma^2_\nuis, \sigma^2_\nuis)
}
where $\sim$ means equality of distributions.
One thus has 
\eq{
	\hat\q_{\nuis^\star}(\M_{\q,\nuis^\star}) - \hat\q_{\nuis}(\M_{\q^\star,\nuis})
		\sim
		\Nn( \q-\q^\star - a_{\nuis^\star} \sigma_{\nuis^\star}^2  + a_{\nuis} \sigma_{\nuis}^2, 
		\sigma_{\nuis^\star}^2 + \sigma_{\nuis}^2 ),
} 
which leads to the results by taking expectation.

}
\bibliographystyle{apa}
\bibliography{references}
\end{document}